%% file: main.tex
\def\isarxiv{1} 

\ifdefined\isarxiv
\documentclass[11pt]{article}

\usepackage[numbers]{natbib}

\else

\documentclass[twoside]{article}

\usepackage{aistats2025}
\usepackage[round]{natbib}

\fi

\usepackage{amsmath}
\usepackage{amsthm}
\usepackage{amssymb}
\usepackage{algorithm}
\usepackage{subfig}
\usepackage{algpseudocode}
\usepackage{graphicx}
\usepackage{grffile}
\usepackage{wrapfig,epsfig}
\usepackage{url}
\usepackage{xcolor}
\usepackage{epstopdf}

\usepackage{bbm}
\usepackage{dsfont}

\allowdisplaybreaks

\ifdefined\isarxiv

\usepackage{tikz}
\usepackage{hyperref}  
\hypersetup{colorlinks=true,citecolor=blue,linkcolor=blue} 
\usetikzlibrary{arrows}
\usepackage[margin=1in]{geometry}

\else

\usepackage{hyperref}

\fi
 
\graphicspath{{./figs/}}

\theoremstyle{plain}
\newtheorem{theorem}{Theorem}[section]
\newtheorem{lemma}[theorem]{Lemma}
\newtheorem{definition}[theorem]{Definition}

\newtheorem{corollary}[theorem]{Corollary}

\newtheorem{assumption}[theorem]{Assumption}

\newtheorem{fact}[theorem]{Fact}

\newcommand{\wh}{\widehat}
\newcommand{\wt}{\widetilde}

\newcommand{\N}{\mathcal{N}}
\newcommand{\R}{\mathbb{R}}

\DeclareMathOperator*{\E}{{\mathbb{E}}}

\DeclareMathOperator{\rank}{rank}
\DeclareMathOperator{\diag}{diag}

\DeclareMathOperator{\tr}{tr}

\DeclareMathOperator*{\argmin}{arg\,min}

\newcommand{\neff}{n_{\mathtt{eff}}} 
\newcommand{\irre}{\mathcal{L}_\mathtt{Irred}}
\newcommand{\appr}{\mathcal{L}_\mathtt{Approx}}
\newcommand{\excess}{\mathcal{L}_\mathtt{Excess}}
\newcommand{\bias}{\mathtt{Bias}}
\newcommand{\vari}{\mathtt{Var}}

\newcommand{\loss}{\mathcal{L}}

\makeatletter
\newcommand*{\RN}[1]{\expandafter\@slowromancap\romannumeral #1@}
\makeatother

\usepackage{lineno}

\begin{document}

\ifdefined\isarxiv

\date{}

\title{
Scaling Law Phenomena Across Regression Paradigms: Multiple and Kernel Approaches
}


\author{
Yifang Chen\thanks{\texttt{
yifangc@uchicago.edu}. The University of Chicago.}
\and
Xuyang Guo\thanks{\texttt{ gxy1907362699@gmail.com}. Guilin University of Electronic Technology.}
\and
Xiaoyu Li\thanks{\texttt{
xiaoyu.li2@student.unsw.edu.au}. University of New South Wales.}
\and
Yingyu Liang\thanks{\texttt{
yingyul@hku.hk}. The University of Hong Kong. \texttt{
yliang@cs.wisc.edu}. University of Wisconsin-Madison.} 
\and
Zhenmei Shi\thanks{\texttt{
zhmeishi@cs.wisc.edu}. University of Wisconsin-Madison.}
\and
Zhao Song\thanks{\texttt{ magic.linuxkde@gmail.com}. The Simons Institute for the Theory of Computing at UC Berkeley.}
}

\else

\twocolumn[

\aistatstitle{
Scaling Law Phenomena Across Regression Paradigms: Multiple and Kernel Approaches
}

\aistatsauthor{ Author 1 \And Author 2 \And  Author 3 }

\aistatsaddress{ Institution 1 \And  Institution 2 \And Institution 3 } ]

\fi

\ifdefined\isarxiv
\begin{titlepage}
  \maketitle
  \begin{abstract}
\input{0_abstract}

  \end{abstract}
  \thispagestyle{empty}
\end{titlepage}

{\hypersetup{linkcolor=black}
\tableofcontents
}
\newpage

\else

\begin{abstract}
\input{0_abstract}
\end{abstract}

\fi

\input{1_intro} 
\input{2_related}
\input{3_preli}

\input{4_main_results}
\input{5_approximation_error}

\input{6_bias_error}
\input{7_excess_error}
\input{8_conclusion}

\ifdefined\isarxiv

\else
\bibliography{ref}
\bibliographystyle{plainnat}
\fi

\newpage
\onecolumn
\appendix

\ifdefined\isarxiv

\begin{center}
    \textbf{\LARGE Appendix }
\end{center}

\else

\aistatstitle{
Exploring the Scaling Law in Kernel Regression: \\
Supplementary Materials}

\fi

\input{9_app_prelim}
\input{10_app_related_work}

\ifdefined\isarxiv
\bibliographystyle{alpha}
\bibliography{ref}
\else

\fi




\end{document}

%% file: 0_abstract.tex
Recently, Large Language Models (LLMs) have achieved remarkable success. A key factor behind this success is the scaling law observed by OpenAI. Specifically, for models with Transformer architecture, the test loss exhibits a power-law relationship with model size, dataset size, and the amount of computation used in training, demonstrating trends that span more than seven orders of magnitude. This scaling law challenges traditional machine learning wisdom, notably the Oscar Scissors principle, which suggests that an overparametrized algorithm will overfit the training datasets, resulting in poor test performance. Recent research has also identified the scaling law in simpler machine learning contexts, such as linear regression. However, fully explaining the scaling law in large practical models remains an elusive goal. In this work, we advance our understanding by demonstrating that the scaling law phenomenon extends to multiple regression and kernel regression settings, which are significantly more expressive and powerful than linear methods. Our analysis provides deeper insights into the scaling law, potentially enhancing our understanding of LLMs.

%% file: 1_intro.tex
\section{Introduction}

The rapid advancement of large language models (LLMs) has revolutionized natural language processing, enabling remarkable generalization across diverse applications~\cite{bmr+20,gpt4}. Transformer-based architectures~\cite{vsp+17}, when scaled to billions of parameters and trained on vast datasets, have proven to be highly effective for numerous NLP tasks. These models, such as BERT~\cite{dclt18}, PaLM~\cite{cnd+23}, Llama~\cite{tli+23}, and GPT-4~\cite{gpt4}, have benefited from methods like multitask fine-tuning~\cite{rsr+20}, prompt tuning~\cite{lac21}, and reinforcement learning from human feedback (RLHF)~\cite{clb+17}, among others. Alongside these empirical advances, scaling laws have been extensively studied~\cite{kmh+20}, revealing the relationship between generalization error and key factors such as compute, sample size, and model parameters~\cite{lwk+24}. These scaling insights have provided guidelines for optimally allocating resources in training deep learning models.

Multiple regression, which predicts vector-valued outputs simultaneously, offers crucial insights for multi-task learning systems. In contrast, kernel regression extends these concepts through non-linear feature mappings that enable more expressive modeling. Recent research has explored both approaches across various regimes, addressing phenomena like double descent and overparameterization~\cite{bhmm19}. Notably, connections between kernel methods and neural networks have been drawn, with insights from the Neural Tangent Kernel (NTK)\cite{jgh18} highlighting similarities between kernel regression and training infinitely wide neural networks.
Thus, it is natural to ask, 

\begin{center}
   {\it Can we extend the error bound from linear regression to multiple and kernel regressions? }
\end{center}

While previous research has established scaling laws for simple linear models~\cite{lwk+24}, the extension to multiple regression settings—where we predict multiple output variables simultaneously—represents a significant advancement in understanding more complex learning paradigms. This work demonstrates that scaling law phenomena extend robustly to multiple regression settings, where we predict vector-valued outputs rather than scalar values. Our analysis reveals that under appropriate assumptions of Gaussianity, well-specified models, and power law decay of eigenvalues in the feature covariance matrix, multiple regression exhibits predictable scaling behaviors similar to those observed in simpler models. This extension to multiple regression settings offers crucial insights for multi-task learning systems. It serves as a bridge between simple linear models and the more expressive kernel methods that we subsequently analyze.

Our primary contributions are as follows:
\begin{itemize}
    \item Deriving the scaling law for multiple regression and kernel regression (see Theorem~\ref{thm:main}).
    \item Presenting novel generalization error bounds for sketched multiple regression (see Theorem~\ref{thm:error_multiregression}).
    \item Presenting novel generalization error bounds for sketched kernel regression (see Corollary~\ref{cor:gaussianity:kernel}).
\end{itemize} 

Our work leverages stochastic gradient descent (SGD) under specific assumptions. These bounds provide insight into the bias-variance trade-off in the sketched predictor setting and help understand the effectiveness of sketching in reducing computational complexity while maintaining predictive accuracy. Furthermore, we derived the scaling law for both the multiple regression and kernel regression based on our generalization error bounds. 

\paragraph{Roadmap.} 
Section~\ref{sec:related_work} discusses the work on scaling laws of Large Language Models and regression problems. 
In Section~\ref{sec:prelim}, we present the notation, definitions, and foundational concepts necessary for understanding their theoretical analysis of multiple and kernel regression.
Section~\ref{sec:main_results} shows our key theoretical findings on scaling laws for multiple regression and kernel regression, including generalization error bounds.
Section~\ref{sec:approx_error} provides detailed mathematical analysis of approximation errors in multiple regression with upper and lower bounds.
In Section~\ref{sec:bias_error}, we derive the bias component of error in multiple regression scenarios with both upper and lower bounds.
In Section~\ref{sec:excess}, we investigate excess errors in multiple regression by examining hypercontractivity, Gaussianity, and well-specified noise under sketched features.
In Section~\ref{sec:conclusion}, we summarize our findings on scaling laws in multiple and kernel regression and discuss implications for understanding generalization in more complex machine learning models.

%% file: 2_related.tex
\section{Related Work}\label{sec:related_work}

\subsection{Scaling Laws}

The scaling behavior of deep learning models concerning compute, sample size, and model size has been a central topic in machine learning research~\cite{hna+17, rrbs19, bmr+20, hkk+20, hbm+22, zkhb22, mrb+23}. A pivotal contribution by~\cite{kmh+20} revealed that generalization error in transformers decays as a power law with respect to these three factors. Their work provided joint formulas predicting how model performance improves with increased compute, data, and parameters. This finding highlighted the importance of scaling models and led to a surge in interest in large-scale models for NLP and other tasks.

Further research refined these scaling laws. For instance,~\cite{hbm+22} proposed the Chinchilla law, which suggested that under compute constraints, balancing data size with model size yields optimal performance, contrasting with earlier models that focused primarily on increasing parameter counts. Meanwhile,~\cite{mrb+23} introduced methods for data reuse across training passes, showing that efficient data handling can maintain performance while lowering the need for excessive compute resources. These empirical studies have provided clearer guidelines for allocating resources effectively when training large models.

In parallel, theoretical work has advanced understanding of these scaling laws.~\cite{sk20} demonstrated that, in regression, the generalization error scales as $n^{-4/d}$, linking model performance to the intrinsic dimensionality of the data. Additionally, other researchers, such as~\cite{bdk+21} and~\cite{bap24}, applied statistical physics to derive scaling laws for linear models, particularly under infinite parameter regimes. Theoretical contributions like these complement empirical findings and offer deeper insights into how models behave as they scale, shaping future directions in machine learning research.

\subsection{Regression Problems}

Regression problems are a very important research topic: they model the dependent variable $Y$ using one or more independent variables $X$. Recent studies have explored various regression problems. One of the most prominent regression problems is linear regression \cite{syz23_quantum,syyz23_ellinf}. Linear regression deals with a single independent variable, whereas multiple regression, $\min_{X \in \R^{d \times N}} \| A X - B \|_F$ \cite{syz23_quantum} given $A \in \R^{n \times d}$ and $B \in \R^{n \times N}$, incorporates two or more predictors to improve prediction accuracy. They have been extensively studied across various domains, including statistics \cite{www+22}, economics \cite{zkh22,an23}, machine learning \cite{sz21,pmp+22,lph+24}, and biomedical sciences \cite{tsg+22,wkc+23}. Other famous regression problems include ridge regression \cite{syz23_quantum,gls+25}, softmax regression \cite{dls23,gsy23_hyper,swy23,lswy23,kls+25}, weighted low-rank approximation \cite{syyz23_weighted,gsyz23,llss25}, leverage scores \cite{lsw+24,lsxy24}, and attention regression \cite{syz23,gsy23_coin,gswy23}. Kernel regression remains a cornerstone in statistical learning, with numerous recent studies delving into its generalization properties in various scenarios.~\cite{cf21} analyzed the generalization error rates, noting the transition between noiseless and noisy regimes.~\cite{mvk24} studied the overfitting tendencies of Gaussian kernel ridgeless regression under changes in bandwidth and dimensionality. 

Additional work by~\cite{mm19} provided precise asymptotic expressions for generalization error in random features regression, highlighting the double descent phenomenon.~\cite{bhmm19} discussed the reconciliation of classical bias-variance trade-off theory with modern machine learning practices, particularly focusing on overparameterization in kernel methods.~\cite{hmrt19} further explored ridgeless least squares interpolation in high dimensions, uncovering unexpected aspects of its generalization performance. The connection between kernel methods and neural networks has also been of interest.~\cite{jgh18} introduced the Neural Tangent Kernel (NTK), showing that training infinitely wide neural networks equates to kernel regression with the NTK, an insight extended by~\cite{ll18,dzps18,adh+19,all19,als19_icml,als19_rnn}, who demonstrated exact computations under the NTK framework.

These studies collectively enhance our understanding of kernel regression's theoretical and practical facets, informing the development of models with improved generalization properties.

%% file: 3_preli.tex
\section{Preliminaries}\label{sec:prelim}
In this section, we introduce some statements to better understand our work. 
In Section~\ref{sec:notations}, we introduce the notation system in our paper. 
Section~\ref{sec:useful_lemma} introduces some important definitions and lemmas used in our paper. This section provides the core definitions and lemmas that formulate our problem statements.
In Section~\ref{sec:linear_regression}, we provide the generalization error of linear regression with and without sketching.
In Section~\ref{sec:multi_regression}, we provide the generalization error of multiple regression with and without sketching.
In Section~\ref{sec:kernel_regression}, we delve into the problem setup and the key theoretical concepts for kernel regression. 
In Section~\ref{sec:sgd_error_decomp}, we define the one-pass SGD for multiple regression and decompose generalization error.

\subsection{Notations}\label{sec:notations}
For any positive integer $n$, we use $[n]$ to denote set $\{1,2,\cdots, n\}$.  
For an event $C$, $\Pr[C]$ represents the probability of event $C$ occurring. $\mathbb{E}[X]$ represents the expected value (or mean) of a random variable $X$. 
For each $a, b \in \R^n$, we use $a \circ b \in \R^n$ to denote the Hadamard product, i.e., the $i$-th entry of $(a\circ b)$ is $a_i b_i$ for all $i \in [n]$.
For $A \in \R^{m \times n}$, let $A_i \in \R^n$ denote the $i$-th row and $A_{*,j} \in \R^m$ denote the $j$-th column of $A$, where $i \in [m]$ and $j \in [n]$.
We use $\|A\|_F$ to denote the Frobenius norm of a matrix $A \in \R^{n \times d}$, i.e., $\|A\|_F := \sqrt{\sum_{i \in [n]} \sum_{j \in [d]} |A_{i,j}|^2}$.
For a symmetric matrix $A \in \R^{n \times n}$, $A \succeq 0$ means that $A$ is positive semidefinite (PSD), i.e., for all $x \in \R^n$, we have $x^\top A x \geq 0$.
For any positive semidefinite matrix $A \in \R^{d\times d}$, for any $w \in \R^d$, we define $\|w\|_A := \sqrt{w^\top A w}$. For a matrix $A$, we use $\lambda_i(A)$ to denote its $i$-th eigenvalue, and we use $\rank(A)$ to denote its rank. 

\subsection{Useful Definitions and Lemmas}\label{sec:useful_lemma}

In this section, we introduce the statements that help understand our work.

We begin with some basic definitions of Gaussian distributions, which form the foundation of our statistical approach.

\begin{definition}[Gaussian Distribution] A Gaussian distribution (or normal distribution) is a continuous probability distribution defined by its mean $\mu $ and variance $ \sigma^2$. The probability density function (PDF) is:
\begin{align*}
    f(x) = \frac{1}{\sqrt{2\pi \sigma^2}} \exp \left( -\frac{(x - \mu)^2}{2\sigma^2} \right).
\end{align*}
\end{definition}

Next, we extend this concept to the multivariate case, essential for analyzing high-dimensional data in our regression models.

\begin{definition}[Multivariate Gaussian Distribution]
A multivariate Gaussian distribution is a generalization of the Gaussian distribution to multiple dimensions. It is defined by a mean vector $ \mu \in \mathbb{R}^d $ and a covariance matrix $ \Sigma \in \mathbb{R}^{d \times d} $. The probability density function (PDF) for a random vector $  x \in \mathbb{R}^d $  is:

\begin{align*}
    f(x) = \frac{1}{\sqrt{(2\pi)^d \det(\Sigma)}} \exp \left( -\frac{1}{2} (x - \mu)^\top \Sigma^{-1} (x - \mu) \right),
\end{align*}
\end{definition}

Then we introduce Isserlis' Theorem, which provides a formula for the expectation of the product of Gaussian random variables, expressed in pairwise expectations.

\begin{lemma}[Isserlis' Theorem, Theorem 1.1, Page 1 in~\cite{mnbo09}]\label{lem:isserlis}
Let $x :=[x_1, x_2, \ldots, x_n] $ be a zero-mean Gaussian random vector. Let $P^2_n$ be the set of all distinct ways of partitioning $\{1, \ldots, n\}$ into pairs.
Then
\begin{align*}
    \E_x [x_1x_2\cdots x_n] = \sum_{p \in P_n^2} \prod_{\{i,j\} \in p} \E_{x} [x_ix_j].
\end{align*}
    
\end{lemma}

\subsection{Linear Regression}\label{sec:linear_regression}

In this section, we introduce the generalization error of linear regression with and without sketching, which provides the theoretical foundation for our later work on multiple regression.

First, we introduce the generalization error of linear regression without sketching, representing the expected loss over the entire data distribution.

\begin{definition}[Generalization Error of Linear Regression, Implicit on Page 4 in~\cite{lwk+24}]
\label{def:linear_regression_error}
    let $P$ be a distribution over $\R^d \times \R$. 
    For each pair $(x,y) \in \R^d  \times \R$, we view $x \in \R^d $ as the feature vector and $y \in \R$ as the label. 
    Then, for a parameter $w \in \R^d $, the population risk is defined as
    \begin{align*}
        \loss(w) :=
        \E_{(x,y)\,\sim\,P}[(\langle w, x\rangle - y)^2].
    \end{align*}
\end{definition}

Next, we present the generalization error of linear regression with sketching, which is crucial for understanding how dimensionality reduction affects the learning performance.

\begin{definition}[General Error of Linear Regression with Sketching]\label{def:linear_regression_sketching}
   Let $P$ be a distribution over $\R^d \times \R$, where $x \in \R^d$ denotes a feature vector and $y \in \R$ is a scalar label. 
    Let $R \in \R^{m \times d}$ be a sketching matrix (e.g., a Gaussian random matrix) whose entries are i.i.d. sampled from $\mathcal{N}(0, 1/m)$.
\begin{align*}
    \E_{(x,y)\sim P}[(\langle v, Rx\rangle- y)^2].
\end{align*}
\end{definition}

Now, we introduce the data covariance and optimal parameter for linear regression, which characterize the statistical properties of our learning problem.

\begin{definition}[Data covariance and optimal parameter for Linear Regression, Definition 1 on Page 4 from~\cite{lwk+24}] We define the data covariance and optimal parameter as follows.
    \begin{itemize}
        \item Let $G := \E[xx^\top]$ be the data covariance.
        \item Assume $\mathrm{tr}(G)$ is finite.
        \item Assume all entries of $G$ are finite.
        \item Let $\{\lambda_i\}_{i \in [d]}$ be the eigenvalues of $G$ sorted in non-increasing order.
        \item Let $w^* \in \arg \min_w \loss(w)$ be the optimal model parameter.
        \item Assume that $\|w^*\|_G^2 := (w^*)^\top Gw^*$ is finite.
    \end{itemize}
\end{definition}

\subsection{Multiple Regression}\label{sec:multi_regression}

This section introduces the generalization error of multiple regression with and without sketching. Unlike linear regression which predicts a single output, multiple regression handles vector-valued outputs simultaneously, making it more suitable for multi-task learning problems.

First, we introduce the generalization error of multiple regression without sketching, quantifying the expected prediction error across all possible inputs.
\begin{definition}[Generalization Error of Multiple Regression]
\label{def:multi_regression_error}
Let $P$ be a distribution over $\R^d \times \R^p$.
For a parameter matrix $W \in \R^{d \times p}$, the generalization error of multiple regression is defined as
\begin{align*}
\loss(W):= 
\E_{(x,y)\sim P}[\|W^\top x - y\|_2^2],
\end{align*}
\end{definition}

Then, we introduce the generalization error of multiple regression with sketching, which applies dimensionality reduction to improve computational efficiency while maintaining good predictive performance.

\begin{definition}[General Error of Multiple Regression with Sketching]
\label{def:multiple_regression_sketching}
    Let $P$ be a distribution over $\R^d \times \R^p$, where $x \in \R^d$ denotes a feature vector and $y \in \R^p$ is the corresponding multi-dimensional label. 
    Let $R \in \R^{m \times d}$ be a sketching matrix, e.g., a Gaussian random matrix whose entries are i.i.d.\ sampled from $\mathcal{N}(0, 1/m)$. 
    For a parameter matrix $V \in \R^{m \times p}$, define the sketched multiple regression by
    \begin{align*}
        f(x) := V^\top(Rx) \in\R^p.
    \end{align*}
    Then the error associated with this multiple regression model is
    \begin{align*}
    \loss(V):=
    \E_{(x,y)\sim P}
    [\|V^\top (R\,x\bigr) - y\|_2^2],
    \end{align*}
\end{definition}

Next, we characterize the statistical properties of our learning problem by defining the data covariance and optimal parameter for multiple regression, which are essential for our theoretical analysis.

\begin{definition}[Data Covariance and Optimal Parameter for Multiple Regression]\label{def:data_cov_multi}
    We define the data covariance and optimal parameter for multiple regression as follows.
    \begin{itemize}
        \item Let the data covariance matrix of $x$ be defined as, 
        \begin{align*}
            G := \E[xx^\top] \in \R^{d\times d}
        \end{align*}
        \item Assume $\mathrm{tr}(G)$ is finite,
        \item Assume all entries of $G$ is finite,
        \item Let $\{ \lambda_i\}_{i \in [d]}$ be the eigenvalues of $G$ in non-increasing order,
        \item Let $W^* \in \R^{d\times p}$ be defined as the optimal parameter:
        \begin{align*}
            W^* \in \arg \min_{W\in \R^{d\times p}} \loss(W) = \arg\min_{W} \E_{(x,y) \sim P} [\|W^\top x-y\|_2^2].
        \end{align*}
        \item Define $\|W\|_G^2 := \mathrm{tr}(W^\top G W)$
    \end{itemize}
\end{definition}

\subsection{Kernel Regression}\label{sec:kernel_regression}

In this subsection, we explore kernel regression, which extends linear models to capture non-linear relationships through feature mappings. This approach allows for more expressive modeling while maintaining mathematical tractability.

We formally define the generalization error of kernel regression as a key metric to evaluate model performance. This measure quantifies how well our model generalizes to unseen data.

\begin{definition}[Generalization Error of Kernel Regression]\label{def:pop_risk:kernel}
    Let $\phi: \R^d \to \R^p$ be a feature mapping.
    Let $P$ be a distribution over $\R^d \times \R$.
    The generalization error of kernel regression is defined as
    \begin{align*}
        \loss(w) := \E_{ (x,y) \sim P }[(\langle w, \phi(x)\rangle - y)^2].
    \end{align*}
\end{definition}

We introduce a sketched version of kernel regression. By applying random sketching, we effectively reduce the number of parameters involved in the model, thus enabling faster training and inference while preserving the essential structure of the problem.

\begin{definition}[Kernel Regression with Sketching]\label{def:regression_sketch:kernel}
    Let $\phi: \R^d \to \R^p$ be a feature mapping.
    Let the kernel predictor $f_v:\R^d \to \R$ be defined by
\begin{align*}
    f_v(x) :=\langle v, R\phi(x)\rangle
\end{align*}
where $v \in \R^m$ and $R \in \R^m \times \R^p$ is a Gaussian sketch matrix where each entry is i.i.d. sampled from $\mathcal{N}(0, 1/m)$.
\end{definition}

Similarly, the generalization error of kernel regression with sketching is defined as follows. This formulation allows us to analyze how dimensionality reduction affects the predictive performance of kernel methods.

\begin{definition}[Generalization Error of Kernel Regression with Sketching]\label{def:pop_risk_sketch:kernel}
We define the following.
\begin{itemize}
    \item Let $\phi: \R^d \to \R^p$ be the feature mapping.
    \item Let $R \in \R^{m \times p}$ be a sketch matrix.
    \item Let $v \in \R^m$ a parameter of a sketched predictor defined as $x \mapsto \langle R \phi(x), v\rangle$.
\end{itemize}

The generalization error of the sketched predictor is defined as
\begin{align*}
    \loss_R(v) := \loss(R^\top v) = \E[(\langle R \phi(x), v \rangle - y)^2].
\end{align*}

\end{definition}

We define the data covariance and optimal parameter as follows, which characterize the statistical properties of our kernel-based learning problem and form the foundation for our theoretical analysis.

\begin{definition}[Data Covariance and Optimal Parameter for Kernel Regression]\label{def:data_cov:kernel}
We define the following.
\begin{itemize}
    \item Let $\Phi := \E_{x \sim P }[\phi(x)\phi(x)^\top]$ be data covariance.  
    \item Assume $\tr[\Phi]$ is finite.
    \item Assume all entries of $\Phi$ are finite. 
    \item Let $\{ \lambda_i \}_{i \in [d]}$ be the eigenvalues of $\Phi$ sorted in non-increasing order. 
    \item Let $W^* \in \argmin_w \loss(w)$ be the optimal model parameter.
    \item Assume that $\|W^*\|_\Phi$ is finite.
\end{itemize}
\end{definition}

\subsection{One-pass SGD and Error Decomposition}\label{sec:sgd_error_decomp}

This section introduces our training methodology and error analysis framework, forming the foundation for scaling law investigation.

The training process of a function $f_v$ uses one-pass Stochastic Gradient Descent (SGD), which is defined as follows. This approach processes each data point exactly once, making it computationally efficient for large datasets.
\begin{definition}[One-Pass SGD for linear regression, Eq. SGD on Page 5 in~\cite{lwk+24}]\label{def:sgd_linear}
Given independent samples $\{(x_1, y_1) , 
 (x_2,y_2), \cdots, (x_n,y_n)\}$ from $P$ and step sizes $\{\gamma_1, \gamma_2, \ldots, \gamma_n\}$.
    The equation of training of $f_v$ via one-pass SGD is defined as
    \begin{align*}
        v_t := &~ v_{t-1} - \gamma_t (f_{v_{t-1}}(x_t) - y_t) \nabla_v f_{v_{t-1}} (x_{t-1}), \quad \forall t \in [n].
    \end{align*}
\end{definition}

Based on Definition~\ref{def:sgd_linear}, we can derive the multiple regression one-pass SGD as follows. This extension allows us to handle vector-valued outputs while maintaining the computational advantages of one-pass training.

\begin{lemma}[One-Pass SGD for Multiple Regression with Sketching]]\label{lem:multi_sgd}
    If the following conditions hold:
    \begin{itemize}
        \item Given samples $\{ (x_1,y_1), (x_2,y_2), \hdots, (x_n,y_n)\}$ from $P$,
        \item Given step sizes $\{ \gamma_1, \gamma_2, \hdots, \gamma_n \}$, one-pass SGD processes each sample once in order,
        \item Let $t \in [n]$,
        \item Define prediction $\wh{y}_t = f_{V_{t-1}}(x_t) = V_{t-1}^\top (Rx_t) \in \R^p$,
        \item Define error $e_t := \wh{y}_t-y_t \in \R^p$.
    \end{itemize}
    Then, the one-pass SGD for Multiple Regression with Sketching is $V_t = V_{t-1} - \gamma_t Rx_t(e_t)^\top$, where $Rx_t(e_t)^\top \in \R^{m\times p}$ is a rank-1 matrix.

\end{lemma}
\begin{proof}
    We first compute the gradient of the loss
    \begin{align}\label{eq:loss_gradient}
        \nabla_V \loss(V) 
        = & ~ \nabla_V\|V^\top (Rx)-y \|_2^2 \notag \\
        = & ~ 2 Rx(V^\top (Rx)-y)^\top
    \end{align}
    where the first step follows from Definition~\ref{def:multiple_regression_sketching} and the second step follows from basic calculus.

    Note that it is common to include a factor of $1/2$ in the loss, then the coefficient of $2$ in Eq.~\eqref{eq:loss_gradient} can be neglected.

    Then, we can prove the one-pass SGD for multiple regression with Sketching.
    \begin{align*}
        V_t 
        & ~ = V_{t-1} - \gamma_t \nabla_V \|V_{t-1}^\top (Rx_t)-y_t \|_2^2 \\
        & ~ = V_{t-1} - \gamma_t  Rx_t(V_{t-1}^\top (Rx_t)-y_t)^\top \\
        & ~ = V_{t-1} - \gamma_t Rx_t(\wh{y}_t-y_t)^\top \\
        & ~ = V_{t-1} - \gamma_t Rx_t(e_t)^\top
    \end{align*}
    where the first step follows from Definition~\ref{def:sgd_linear} and Definition~\ref{def:multiple_regression_sketching}, the second step follows from Eq.~\eqref{eq:loss_gradient}, the third step follows from the lemma conditions, and the fourth step follows from the lemma conditions.
    It is natural to see $Rx_t(e_t)^\top \in \R^{m\times p}$ is a rank-1 matrix.
    
    Thus, we complete the proof.
\end{proof}

Next, we provide a decomposition of the resulting generalization error. This decomposition allows us to analyze the different sources of error present in the model's performance.

\begin{definition}[Generalization Error Decomposition]\label{lem:risk_decompose}
After $n$ gradient steps, we get $V_n$. The error on the sample set $R = \{ (x_i,y_i) \}$ is defined as follows.
\begin{align*}
    \loss_R(V_n) = & ~ \underbrace{\min_V \loss(V)}_{\irre} + \underbrace{\min_V \loss_R(V) - \min_V \loss(V)}_{\appr} + \underbrace{\loss_R(V_n) - \min_V \loss_R(V)}_{\excess},
\end{align*}
where $\irre$ is the Irreducible Error, $\appr$ is the Approximation Error, and $\excess$ is the Excess Error.    
\end{definition}

%% file: 4_main_results.tex
\section{Main Results}\label{sec:main_results}

In this section, we introduce our main results.
In Section~\ref{sec:scaling_law_multiple}, we present the scaling law and its generalization error bounds for multiple regression. 
In Section~\ref{sec:scaling_law_kernel}, we present the scaling law and its generalization error bounds for kernel regression.

\subsection{Scaling Law for Multiple Regression}\label{sec:scaling_law_multiple}

In this section, we present the scaling law for multiple regression, which forms the core theoretical contribution of our work. 

We first introduce the assumptions necessary to establish our results, each providing important constraints that enable our mathematical analysis.

\begin{assumption}[Gaussianity]\label{ass:gaussian_feature}
    Assume that $(x, y)$ drawn from $P$ satisfies $x \sim \N(0, G)$.
\end{assumption}

Assumption \ref{ass:gaussian_feature} establishes the distributional properties of the feature mapping, which will facilitate subsequent analysis under the Gaussian framework. 

\begin{assumption}[Well-specified Multiple Regression Model]\label{ass:well_specified}
    Assume that $(x, y)$ drawn from $P$ satisfy $\E_{ (x,y) \sim P }[y~|~x]=W^{*\top}x$. Define $\sigma^2 := \E_{(x,y)\sim P}[\|y-W^{*\top}x \|_2^2]$.
\end{assumption}

Assumption \ref{ass:well_specified} introduces the well-specified nature of the feature mapping, ensuring that the target $y$ is conditionally linear in $\phi(x)$, which aids in specifying the true parameter $W^*$.

\begin{assumption}[Parameter Prior]\label{ass:parameter_prior}
    Assume that $W^*$ satisfies a prior such that $\E[W^* W^{*\top}] = I_{d\times d}$.
\end{assumption}

This prior assumption on $W^*$ ensures that the covariance of the parameter vector is well-structured, which will be critical for deriving the posterior distribution.

\begin{assumption}[Power Law Spectrum]\label{ass:power_law_spectrum:kernel}
    There exists $a > 1$ such that the eigenvalues of $G$ satisfy $\lambda_i = \Theta(i^{-a}), i > 0$.
\end{assumption}

Assumption \ref{ass:power_law_spectrum:kernel} constrains the spectrum of the feature covariance matrix, implying a power law decay for the eigenvalues, which characterizes the complexity of the feature space.

Building upon the established assumptions and setup, we present our main theoretical result regarding the generalization error bounds under the sketched predictor scenario.

\begin{theorem}[Generalization Error Bounds for Multiple Regression Scaling Law]\label{thm:error_multiregression}
If the following conditions hold:
\begin{itemize}
    \item Suppose that Assumptions~\ref{ass:gaussian_feature}, \ref{ass:well_specified}, \ref{ass:parameter_prior}, and \ref{ass:power_law_spectrum:kernel} hold.
    \item Consider an $m$-dimensional sketched predictor trained by (SGD Lemma~\ref{lem:multi_sgd}) with $n$ samples.
     \item Let $\neff := n/ \log(n )$ and recall the generalization error decomposition in Lemma~\ref{lem:risk_decompose}.
    \item Let $D = \{(x_1, y_1), (x_2, y_2), \ldots, (x_n, y_n)\}$ be the sample set i.i.d. drawn from $P$.
    \item Let the initial stepsize $\gamma \leq c$ for some $a$-dependent constant $c > 0$.
\end{itemize}   
   Then with probability at least $1 - e^{-\Omega(M )}$ over the randomness of the sketch matrix $R$, we have
\begin{itemize}
    \item Part 1. $\irre = \sigma^2$.
    \item Part 2. $\E_{W^*}[\appr] = \Theta(m^{1-a})$.
\end{itemize}
    If we further assume that $\sigma^2 = \Omega(1)$, then the following holds
\begin{itemize}
    \item Part 3. $  \E_{W^*}[\E_{D \sim P^n}[\excess]] = \Theta (\bias + \sigma^2 \vari). $   
\end{itemize}
\end{theorem}

\begin{proof}
    {\bf Part 1.} 
    First we show,
    \begin{align*}
        \loss(W) 
        = & ~ \E[\| W^\top x - y \|_2^2] \\
        = & ~ \E[ \| (W^\top - W^{*\top})x + (W^{*\top}x-y) \|_2^2 ] \\
        = & ~ \E[ \| (W^\top - W^{*\top})x \|_2^2 + \|(W^{*\top}x-y) \|_2^2 ] \\
        \geq & ~ \sigma^2
    \end{align*}
    where the first step follows from Definition~\ref{def:multiple_regression_sketching}, the second step follows from basic algebra, the third step follows from the definition of $\ell_2$-norm, and the fourth step follows from Assumption~\ref{ass:well_specified} and $\ell_2$-norm $\geq 0$.

    {\bf Part 2.} Follows from Lemma~\ref{lem:approx_error_lower_bound_mutliregression}.

    {\bf Part 3.} Follows from Lemma~\ref{lem:excee_risk_bounds}.
    
\end{proof}

Theorem \ref{thm:error_multiregression} provides a comprehensive characterization of the generalization performance, revealing how the error decomposes based on various structural assumptions and conditions of the predictor. This result highlights key dependencies on parameters such as the sketch dimension $M$, the effective sample size $N_{\rm eff}$, and the feature spectrum.

Leveraging the generalization error bounds established in Theorem \ref{thm:error_multiregression}, we derive a scaling law that provides further insights into how the error components behave under different parameter regimes.

\begin{theorem}[Scaling Law for Multiple Regression]\label{thm:main}
    Assume all conditions in Theorem~\ref{thm:error_multiregression} hold. Then with probability at least $1 - e^{-\Omega(M )}$ over the randomness of the sketch matrix $R$ that
    \begin{align*}
      \E[\loss_R(V_n)]] = \sigma^2 + \Theta(\frac{1}{M^{a-1}}) + \Theta(\frac{1}{(N_{\rm eff} \gamma)^{(a-1)/a}}).
    \end{align*}
\end{theorem}

\begin{proof}
    Adding all the errors in Theorem~\ref{thm:error_multiregression}, we can derive the scaling law.
\end{proof}

Theorem \ref{thm:main} reveals the interplay between the sketch dimension $M$, the effective sample size $\neff$, and the learning rate $\gamma$ in determining the overall loss, offering a practical guideline for scaling behavior in sketched predictors.

\subsection{Scaling Law for Kernel Regression}\label{sec:scaling_law_kernel}

This section provides the generalization error bounds for kernel regression scaling law. We just need to replace the sketching $\|V^\top (Rx) - y \|_2^2$ from Definition~\ref{def:multiple_regression_sketching} with $(\langle R\phi(x),v\rangle-y)^2$ in proofs of Lemma~\ref{lem:approximation_error:multiregression},~\ref{lem:approx_error_upper_bound_mutliregression},~\ref{lem:approx_error_lower_bound_mutliregression},~\ref{lem:bias_error_upper_bound_mutliregression} ,~\ref{lem:bias_error_lower_bound_mutliregression},~\ref{lem:hypercontractivity and misspecified noise},~\ref{lem:gaussianity and noise},~\ref{lem:excee_risk_bounds}.

\begin{corollary}[Generalization Error Bounds for Kernel Regression Scaling Law]\label{cor:gaussianity:kernel}
If the following conditions hold:
\begin{itemize}
    \item Assume that $(x, y)$ drawn from $P$ and the feature mapping $\phi$ satisfy $\phi(x) \sim \N(0, \Phi)$.
    \item Assume that $(x, y)$ drawn from $P$ and the feature mapping $\phi$ satisfy $\E_{ (x,y) \sim P }[y~|~x]=\phi(x)^\top w^*$. Define $\sigma^2 := \E_{(x,y)\sim P}[(y-\phi(x)^\top w^*)^2]$.
\end{itemize}
Then the following statements hold:
\begin{itemize}
    \item Part 1. $R\phi(x) \sim \N(0, R\Phi  R^\top)$.
    \item Part 2. $\E [y ~ |~ R \phi(x)] = \langle R\phi(x),V^*\rangle$.
    \item Part 3. $\E[(y- \langle R\phi(x),V^*\rangle)^2] = \sigma^2 + \appr \geq \sigma^2$.
\end{itemize}
\end{corollary}
\begin{proof}[Proof Sketch]
    Proof is similar to the multiple regression proof in Theorem~\ref{thm:error_multiregression} by using Definition~\ref{def:regression_sketch:kernel}.
\end{proof}

%% file: 5_approximation_error.tex
\section{Approximation Error Bounds for Multiple Regression}\label{sec:approx_error}
In this section, we prove and bound the approximation error for multiple regression.
In Section~\ref{sec:approx_multi}, we prove the approximation error for multiple regression.
Section~\ref{sec:approx_multi_upper} shows the upper bound of approximation error for multiple regression.
Section~\ref{sec:approx_multi_lower} presents the lower bound result of approximation error for multiple regression.

\subsection{Approximation Error for Multiple Regression}\label{sec:approx_multi}

In this section, we formulate the approximation error for multiple regression, quantifying the loss incurred due to the dimensional reduction through sketching. This error is critical in understanding the trade-offs between computational efficiency and prediction accuracy.

We establish a formal characterization of the approximation error, which measures the difference between the minimum achievable loss with and without sketching constraints.

\begin{lemma}[Approximation Error]\label{lem:approximation_error:multiregression}
    If the following conditions hold
    \begin{itemize}
        \item Let $G \in \R^{d\times d}$ be defined in Definition~\ref{def:data_cov_multi}.
        \item Let $R \in \R^{m \times d}$ be a Gaussian sketch matrix.
        \item $V^* := (R G R^\top)^{-1}RG W^*$.
        \item $\appr := \min \loss_R(\cdot) - \min \loss(\cdot)$.
    \end{itemize}   
    Then the following statements hold (without using randomness of $R$)
    \begin{itemize}
        \item Part 1. The minimizer of $\loss_R(V)$ is $V^*$
    \item Part 2. We can show
    \begin{align*}
        \appr = \|(I - G^{1/2}R^\top(R G R^\top)^{-1}RG^{1/2})G^{1/2}W^*\|_F^2.
    \end{align*}
    \end{itemize}
    Moreover, the following statement holds
    \begin{itemize}
        \item Part 3. $\appr \leq \|W^*\|_G^2$ almost surely over the randomness of $R$.
    \end{itemize}
\end{lemma}

\begin{proof}
    Recall that the generalization error 
    \begin{align*}
        \loss(W):= 
        \E_{(x,y)\sim P}[\|W^\top x - y\|_2^2],
    \end{align*}
    is a quadratic function and that $W^*$ is the minimizer of $\loss(\cdot)$, so we have
    \begin{align*}
        \E[xx^\top] W^* = \E[xy]
    \end{align*}
    and it is equivalent to 
    \begin{align*}
        G W^* = \E[xy]
    \end{align*}
    where it follows from setting the gradient of $\loss$ with respect to $W$ to 0,
    and
    \begin{align}
    \label{eq:diff_rw_rw*:kernel}
        \loss(W) = &~\E[\|W^\top x - W^{*\top}x\|_2^2] + \loss(W^*) \notag \\
        =&~\|G^{1/2}(W-W^*)\|_F^2 + \loss(W^*)
    \end{align}
    where the first step is due to basic algebra, and the second step uses $\E[xx^\top] = G$.

    Recall that the generalization error in a restricted subspace
    \begin{align*}
        \loss_R(V) = \loss(V^\top R) = \E[\|(V^\top R)x - y\|_2^2]
    \end{align*}
    is also a quadratic function.

    Its minimizer $V^*$ is given as follows.
    \begin{align}\label{eq:v*:kernel}
        V^* = &~(R G R^\top)^{-1}RG W^*
    \end{align}
    Thus, we can show that
    \begin{align*}
        \appr = &~ \loss_R(V^*) - \loss(W^*) \\
        = &~ \loss(V^{*\top}R) - \loss(W^*) \\
        = &~ \|G^{1/2}(R^\top V^* - W^*)\|_F^2 \\
        = &~ \|G^{1/2}(R^\top(RG R^\top)^{-1}RG W^*-W^*)\|_F^2 \\
        =& \| (I - G^{1/2}R^\top(RG R^\top)^{-1}RG ^{1/2})G^{1/2}W^*\|_F^2
    \end{align*}
    where the first step follows from the definition of $\appr$, the second step follows from the definition of $\loss_S$, the third step uses the Eq.~\eqref{eq:diff_rw_rw*:kernel}, where the fourth step follows from Eq.~\eqref{eq:v*:kernel}, and the last step uses basic algebra.
    
    Furthermore, we show that
    \begin{align*}
        & ~ (I - G^{1/2}R^\top(RG R^\top)^{-1}RG ^{1/2})^2\\
        = &~ I - G^{1/2}R^\top(RG R^\top)^{-1}RG ^{1/2} \\
        \preceq &~ I,
    \end{align*}
    where the first and second step follow from basic algebra.
    
    Hence by definition of $\|\cdot \|_G $, it follows that $\appr \leq \|W^*\|_G^2$.
\end{proof}

This lemma provides several important insights into the approximation error. First, it identifies the optimal parameter $V^*$ in the sketched space as a transformed version of the original optimal parameter $W^*$. Second, it gives an exact expression for the approximation error in terms of the projection operators defined by the sketching matrix and data covariance. 

The formulation allows us to analyze how the sketch dimension affects the approximation quality, which is crucial for understanding the scaling behavior in our later results. As we will see, this error term exhibits a power-law decay with respect to the sketch dimension under our assumptions.

\subsection{Upper Bound of Approximation Error}\label{sec:approx_multi_upper}

With the approximation error from Lemma~\ref{lem:approximation_error:multiregression} established, we now proceed to derive an upper bound on this error term. This bound is crucial for understanding how the approximation quality scales with the sketch dimension and problem parameters.

The following lemma provides a detailed characterization of the upper bound, breaking it down into components that capture the interplay between the power law decay of eigenvalues and the sketching dimension.

\begin{lemma}[Multiple Regression Approximation Error Upper bound]\label{lem:approx_error_upper_bound_mutliregression}
    Suppose $r(G) \geq k+M$ for some $k\leq d$, for a sketching matrix $R \in \R^{m \times d}$ (see Definition~\ref{def:multiple_regression_sketching}) and Let $\tau \leq p$, the approximation error is defined as follows.
    \begin{align*}
        \mathsf{Approx} := \min_{V\in \R^{m\times p}} \| (V^\top R - W^{*\top})G^{1/2} \|_F^2
    \end{align*}. Let $k \leq M/2$, then with probability $1 - e^{-\Omega(M)}$, we have    
    \begin{align*}
        \mathsf{Approx} = O(\|W_{k:\tau}^*\|_G^2 + (\sum_{i > k} \frac{\lambda_i}{M} + \lambda_{k+1} + \sqrt{\frac{\sum_{i>k} \lambda_i^2}{M}} ) \|W_{0:k}^*\|_2^2)
    \end{align*}
\end{lemma}

\begin{proof}
    With the singular value decomposition $G = U \Lambda U^\top$, where $\Lambda := \diag \{ \lambda_1, \lambda_2, \hdots \}$, we let $\lambda_1 \geq \lambda_2 \geq \hdots \geq 0$ and $UU^\top = I$.

    We define 
    $\wt{R} := RU \Lambda^{1/2}$ and $\wt{W}^* := \Lambda^{1/2} U^\top W^*$.

    Then, we can show:
    \begin{align}\label{eq:approx}
        \mathsf{Approx}(R,G,W^*) \
        = & ~ \| (I-G^{1/2}R^\top (RGR^\top)^{-1}RG^{1/2})G^{1/2}W^* \|_F^2 \notag \\
         = & ~ \| (I-U\Lambda^{1/2}\wt{R}^\top (\wt{R}\Lambda\wt{R}^\top)^{-1}R\Lambda^{1/2}U^\top)U\Lambda^{1/2}U^\top W^* \|_F^2 \notag \\
         = & ~ \| U(I - \Lambda^{1/2}\wt{R}^\top (\wt{R}\Lambda \wt{R}^\top)^{-1}R\Lambda^{1/2})\Lambda^{1/2}U^\top W^* \|_F^2 \notag \\
         = & ~ \| ( I - \Lambda^{1/2}\wt{R}^\top(\wt{R}\Lambda\wt{R}^\top)^{-1}R\Lambda^{1/2} ) \Lambda^{1/2} \wt{W}^* \|_F^2 \notag \\
         = & ~ \mathsf{Approx}(\wt{R}, \Lambda, \wt{W}^*)
    \end{align}
    where the first step follows from Lemma~\ref{lem:approximation_error:multiregression}, the second step follows from the singular value decomposition, the third step follows from basic algebra, the fourth step follows from the singular value decomposition, and the last step follows from the definition of $\mathsf{Approx}$.

    Then, we analyze the case where $G = \Lambda$ is a diagonal matrix with non-increasing diagonal entries. We define $\underbrace{A}_{m\times m} := \underbrace{R}_{m\times d}\underbrace{G}_{d\times d}R^\top$

    By definition of $\mathsf{Approx}$, we have
    \begin{align*}
        \mathsf{Approx} 
        = & ~ \| (I-G^{1/2}R^\top (RGR^\top)^{-1}RG^{1/2})G^{1/2}W^* \|_F^2 \notag \\
        = & ~ \langle (G^{1/2} R^\top A^{-1} RG^{1/2} - I_{d\times d}) (G^{1/2} R^\top A^{-1} RG^{1/2} - I_{d\times d})^\top, G^{1/2}W^*W^{*\top} G^{1/2} \rangle
    \end{align*}
    where the first step follows from Eq.~\eqref{eq:approx}, and the second step follows from the properties of Frobenius norm and inner products.

    For $k \in [p]$, we have
    \begin{align*}
        & ~ \underbrace{G^{1/2}}_{d\times d} \underbrace{R^\top}_{d\times m} \underbrace{A^{-1}}_{m\times m} \underbrace{RG^{1/2}}_{m\times d} - I_{d\times d} \\
        = & ~ \begin{bmatrix}
            G_{0:k}^{1/2} R_{0:k}^\top\\ 
             G_{k:\tau}^{1/2} R_{k:\tau}^\top \\ 
        \end{bmatrix} A^{-1} 
        \begin{bmatrix}
            R_{k:\tau} G_{k:\tau}^{1/2} & 
             R_{k:\tau} G_{k:\tau}^{1/2} \\ 
        \end{bmatrix}
        -  I_{d\times d} \\ 
        = & ~ 
        \begin{bmatrix}
            G_{0:k}^{1/2} R_{0:k}^\top A^{-1}R_{0:k}G_{0:k}^{1/2}-I_{k\times k} & 
             G_{0:k}^{1/2} R_{0:k}^\top A^{-1}R_{k:\tau}G_{k:\tau}^{1/2} \\ 
             G_{k:\tau}^{1/2}R_{k:\tau}^\top A^{-1} R_{0:k} G_{0:k}^{1/2} & G_{k:\tau}^{1/2}R_{k:\tau}^\top A^{-1} R_{k:\tau}G_{k:\tau}^{1/2}- I_{(d-k)\times (d-k)}
        \end{bmatrix}
    \end{align*}
    where the first step follows from the definition of matrix, and the second step follows from basic matrix algebra.

    Based on the pattern, we define 
    \begin{align}\label{eq:redefined_matrix}
        \begin{bmatrix}
            \alpha & \beta \\
            \beta^\top & \kappa
        \end{bmatrix}
        :=
        \begin{bmatrix}
            G_{0:k}^{1/2} R_{0:k}^\top A^{-1}R_{0:k}G_{0:k}^{1/2}-I_{k\times k} & 
             G_{0:k}^{1/2} R_{0:k}^\top A^{-1}R_{k:\tau}G_{k:\tau}^{1/2} \\ 
             G_{k:\tau}^{1/2}R_{k:\tau}^\top A^{-1} R_{0:k} G_{0:k}^{1/2} & G_{k:\tau}^{1/2}R_{k:\tau}^\top A^{-1} R_{k:\tau}G_{k:\tau}^{1/2}- I_{(d-k)\times (d-k)}
        \end{bmatrix}
    \end{align}
    Then, we can have
    \begin{align}\label{eq:inequality_approx}
        & ~ (G^{1/2} R^\top A^{-1} RG^{1/2} - I_{d\times d}) (G^{1/2} R^\top A^{-1} RG^{1/2} - I_{d\times d})^\top \notag \\
        = & ~ \begin{bmatrix}
            \alpha^2 + \beta \beta^\top & \alpha\beta + \beta \kappa \\
            \beta^\top \alpha + \kappa\beta^\top & \kappa^2 + \beta^\top \beta
        \end{bmatrix} \notag \\
        \preceq & ~ 2 \begin{bmatrix}
            \alpha^2 + \beta \beta^\top & 0\\
            0 & \kappa^2+\beta^\top \beta
        \end{bmatrix}
    \end{align}
    where the first step follows from Eq.~\eqref{eq:redefined_matrix} and matrix multiplication, and the second step follows from PSD inequality. 

    Thus, we have
    \begin{align*}
        \mathsf{Approx} 
        & ~ \leq 2 \langle \begin{bmatrix}
            \alpha^2 + \beta \beta^\top & 0\\
            0 & \kappa^2+\beta^\top \beta
        \end{bmatrix}, G^{1/2}W^*W^{*\top}G^{1/2} \rangle \\
        & ~ = 2 \langle \alpha^2 + \beta \beta^\top, G_{0:k}^{1/2}W_{0:k}^* W_{0:k}^{*\top} G_{0:k}^{1/2} \rangle + 2 \langle \kappa^2 + \beta^\top \beta, G_{k:\tau}^{1/2}W_{k:\tau}^* W_{k:\tau}^{*\top}G_{k:\tau}^{1/2} \rangle.
    \end{align*}
    where the first step follows from Eq.~\eqref{eq:inequality_approx}, and the second step from basic properties of inner product.

    We further show that
    \begin{align*}
        -I_{(d-k)\times (d-k)} 
        \preceq & ~W \\
        = & ~ G_{k:\tau}^{1/2}R_{k:\tau}^\top A^{-1} R_{k:\tau}G_{k:\tau}^{1/2}- I_{(d-k)\times (d-k)}\\
        = & ~ G_{k:\tau}^{1/2}R_{k:\tau}^\top (R_{0:k}G_{0:k}R_{0:k}^\top + R_{k:\tau}G_{k:\tau}R_{k:\tau}^\top)^{-1} R_{k:\tau}G_{k:\tau}^{1/2}- I_{(d-k)\times (d-k)}\\
        \preceq & ~
         G_{k:\tau}^{1/2} R_{k:\tau}^\top (R_{k:\tau} G_{k:\tau} R_{k:\tau}^\top)^{-1} R_{k:\tau} G_{k:\tau}^{1/2} - I_{(d-k)\times (d-k)} \\
        \preceq & ~ 0_{d\times (d-k)}
    \end{align*}
    where the first step and second step follows from the definition of $\kappa$ from Eq.~\eqref{eq:redefined_matrix}, the third step follows from the definition of $A$, the fourth step follows from basic matrix algebra, and the fifth step follows from the fact that $\| \kappa \|_2 \leq 1$.

    Then, we need to show
    \begin{align}\label{eq:1st_claim}
        \kappa^2 + \beta^\top \beta = -\kappa, 
    \end{align}
    By the definition of $\kappa$ from Eq.~\eqref{eq:redefined_matrix}, we can have
    \begin{align*}
        \kappa^2 
        = & ~ (G_{k:\tau}^{1/2} R_{k:\tau}^\top G^{-1} R_{k:\tau} G_{k:\tau}^{1/2} - I_{(d-k)\times (d-k)})^2 \\
        = & ~ I_{(d-k)\times (d-k)} - 2G_{k:\tau}^{1/2} R_{k:\tau}^\top G^{-1} R_{k:\tau} G_{k:\tau}^{1/2} + G_{k:\tau}^{1/2} R_{k:\tau}^\top G^{-1} R_{k:\tau} G_{k:\tau} R_{k:\tau}^\top G^{-1} R_{k:\tau} G_{k:\tau}^{1/2} \\
        = & ~ I_{(d-k)\times (d-k)} - 2G_{k:\tau}^{1/2} R_{k:\tau}^\top G^{-1} R_{k:\tau} G_{k:\tau}^{1/2} +  G_{k:\tau}^{1/2} R_{k:\tau}^\top G^{-1} G_k G^{-1} R_{k:\tau} G_{k:\tau}^{1/2}.
    \end{align*}
    where the first step follows from Eq.~\eqref{eq:redefined_matrix}, the second step follows from basic matrix algebra, and the third step follows from basic matrix algebra.

    By definition of $\beta$ from Eq.~\eqref{eq:redefined_matrix}, we have the following.
    \begin{align*}
        \beta^\top \beta = G_{k:\tau}^{1/2} R_{k:\tau}^\top G^{-1} (R_{0:k} G_{0:k} R_{0:k}^\top) G^{-1}  R_{k:\tau} G_{k:\tau}^{1/2}.
    \end{align*}
    Then, we can show
    \begin{align*}
        \kappa^2 + \beta^\top \beta 
        = & ~I_{(d-k)\times (d-k)} - 2G_{k:\tau}^{1/2} R_{k:\tau}^\top G^{-1} R_{k:\tau} G_{k:\tau}^{1/2} +  G_{k:\tau}^{1/2} R_{k:\tau}^\top G^{-1} G_k G^{-1} R_{k:\tau} G_{k:\tau}^{1/2} \\
        = & ~I_{(d-k)\times (d-k)} - G_{k:\tau}^{1/2} R_{k:\tau}^\top G^{-1} R_{k:\tau} G_{k:\tau}^{1/2} \\
        = & ~ -\kappa.
    \end{align*} 
    where the first step follows from the fact that $R_{0:k} G_{0:k} R_{0:k}^\top + A_k = A$, the second step follows from basic matrix algebra, and the third step follows from the definition of $\kappa$ in Eq.~\eqref{eq:redefined_matrix}. 
    Thus, we have proved Eq.~\eqref{eq:1st_claim}.

    We further prove the following.
    \begin{align}\label{eq:2nd_claim}
        \langle \alpha^2 + \beta\beta^\top, G_{0:k}^{1/2} W_{*,0:k} W_{*,0:k}^\top G_{0:k}^{1/2} \rangle = \langle (G_{0:k}^{-1} + R_{0:k}^\top A_k^{-1} R_{0:k})^{-1}, W_{0:k}^* W_{0:k}^{*\top} \rangle.
    \end{align}
    First, we show.
    \begin{align*}
        \alpha
        = & ~ G_{0:k}^{1/2} R_{0:k}^\top A^{-1} R_{0:k} G_{0:k}^{1/2} - I_k \\
        = & ~ G_{0:k}^{1/2} R_{0:k}^\top A_k^{-1} R_{0:k} G_{0:k}^{1/2} - G_{0:k}^{1/2} R_{0:k}^\top A_k^{-1} R_{0:k}[G_{0:k}^{-1} + R_{0:k}^\top A_k^{-1} R_{0:K}]^{-1} R_{0:k}^\top A_k^{-1} R_{0:K} G_{0:k}^{1/2} - I_k \\
        = & ~ G_{0:k}^{1/2} R_{0:k}^\top A_k^{-1} R_{0:k} [G_{0:k}^{-1} + R_{0:k}^\top A_k^{-1} R_{0:K}]^{-1} G_{0:k}^{-1} G_{0:k}^{1/2} - I_k \\
        = & ~ G_{0:K}^{1/2} (R_{0:k}^\top A_{k}^{-1} R_{0:k}[G_{0:k}^{-1} + R_{0:k}^\top A_k^{-1} R_{0:k}]^{-1} - I_k) G_{0:k}^{-1/2} \\
        = & ~ -G_{0:k}^{-1/2}[G_{0:k}^{-1} + R_{0:k}^\top A_k^{-1} R_{0:k}]^{-1} G_{0:k}^{-1/2}.
    \end{align*}
    where the first step follows from Eq.~\eqref{eq:redefined_matrix}, the second step follows from Woodbury's matrix identity ($A^{-1} = [R_{0:k} G_{0:k} R_{0:k}^\top + A_k]^{-1} = A_k^{-1} - A_k^{-1} R_{0:k} [G_{0:k}^{-1} + R_{0:k}^\top A_k^\top R_{0:k}]^{-1} R_{0:k}^\top A_k^{-1}$), the third step follows from basic matrix algebra, the fourth step follows from the associative rule, and the fifth step follows from basic matrix algebra.

    Thus, we can have the following,
    \begin{align}\label{eq:alpha_squre}
        \alpha^2 = G_{0:k}^{-1/2}(G_{0:k}^{-1} + R_{0:k}^\top A_k^{-1} R_{0:k})^{-1} G_{0:k}^{-1}(G_{0:k}^{-1} + R_{0:k}^\top A_k^{-1} R_{0:k})^{-1} G_{0:k}^{-1/2} 
    \end{align}
    We define $B = (G_{0:k}^{-1} + R_{0:k}^\top A_k^{-1} R_{0:k})$
    Then, we show that
    \begin{align}\label{eq:beta_beta}
        \beta \beta^\top 
        = & ~ G_{0:k}^{1/2} R_{0:k}^\top A^{-1} R_{k:\tau} G_{k:\tau} R_{k:\tau}^\top A^{-1} R_{0:k} G_{0:k}^{1/2} \notag \\
        = & ~ G_{0:k}^{-1/2} B^{-1} R_{0:k}^\top A_k^{-1} (R_{k:\tau} G_{k:\tau} R_{k:\tau}^\top) A_k^{-1} R_{0:k} B^{-1} G_{0:k}^{-1/2} \notag \\
        = & ~ G_{0:k}^{-1/2} B^{-1} R_{0:k}^\top A_k^{-1} R_{0:k} B^{-1} G_{0:k}^{-1/2}
    \end{align}
    where the first step follows from the definition of $\beta$ in Eq.~\eqref{eq:redefined_matrix}, the second step follows from Woodbury's matrix identity, and the third step follows from basic matrix algebra.

    Then, we have the following by adding Eq.~\eqref{eq:alpha_squre} with Eq.~\eqref{eq:beta_beta} and the definition of $B$.
    \begin{align}\label{eq:u_vv}
        \alpha^2 + \beta \beta^\top = G_{0:k}^{-1/2} (G_{0:k}^{-1} + R_{0:k}^\top A_k^{-1} R_{0:k})^{-1}G_{0:k}^{-1/2},  
    \end{align}
    Thus, by the definition of inner product, we have,
    \begin{align*}
        \langle \alpha^2 + \beta \beta^\top, G_{0:k}^{1/2} W_{*,0:k} W_{*,0:k}^\top G_{0:k}^{1/2} \rangle = \langle (G_{0:k}^{-1} + R_{0:k}^\top A_k^{-1} R_{0:k})^{-1}, W_{0:k}^* W_{0:k}^{*\top} \rangle.
    \end{align*}
    By Lemma 26 in~\cite{bllt20}, with probability $1-e^{-\Omega(M)}$, we have
    \begin{align}\label{eq:mu}
        \mu_{\min} (A_k^{-1}) &= \| A_k \|_2^{-1} \notag \\
        &= c / (\frac{\sum_{i>k} \lambda_i}{M} + \lambda_{k+1} + \sqrt{\frac{\sum_{i>k} \lambda_i^2}{M}} )
    \end{align}
    By Theorem 6.1 in~\cite{w19}, we have $R_{0:k}^\top R_{0:k} \succeq I_{k/5}$ with probability $1-e^{-\Omega(M)}$ when $M/k \geq 2$. Then, we have
    \begin{align}\label{eq:s_a_s}
        R_{0:k}^\top A_k^{-1} R_{0:k} \succeq & ~ c R_{0:k}^\top R_{0:k} / (\frac{\sum_{i>k} \lambda_i}{M} + \lambda_{k+1} + \sqrt{\frac{\sum_{i>k} \lambda_i^2}{M}} ) \notag \\
        = & ~ \Omega(I_k/ (\frac{\sum_{i>k} \lambda_i}{M} + \lambda_{k+1} + \sqrt{\frac{\sum_{i>k} \lambda_i^2}{M}} ))
    \end{align}
    where the first step follows from Eq.~\eqref{eq:mu}, and the second step follows from the definition of time complexity.

    Thus, we can have
    \begin{align*}
        \langle (G_{0:k}^{-1} + R_{0:k}^\top A_k^{-1} R_{0:k})^{-1}, W_{0:k}^* W_{0:k}^{*\top} \rangle 
        \leq & ~ \langle (R_{0:k}^\top A_k^{-1} R_{0:k})^{-1}, W_{0:k}^* W_{0:k}^{*\top} \rangle \\
        \leq & ~ \| (G_{0:k}^{-1} +R_{0:k}^\top A_k^{-1} R_{0:k})^{-1}\|_2 \|W_{0:k}^*\|_F^2 \\
        = & ~ O((\frac{\sum_{i>k} \lambda_i}{M} + \lambda_{k+1} + \sqrt{\frac{\sum_{i>k} \lambda_i^2}{M}} )) \|W_{0:k}^*\|_F^2)
    \end{align*}
    where the first step follows from basic algebra, the second step follows from the definition of inner product, and the third step follows from Eq.~\eqref{eq:s_a_s}.

    Thus, we complete the proof.    
\end{proof}

This upper bound reveals several important insights about the approximation error. First, it separates the contribution from the tail parameters $W_{k:\tau}^*$ and the head parameters $W_{0:k}^*$. The error from the tail parameters is directly proportional to their energy under the covariance metric. For the head parameters, the error depends on their magnitude scaled by factors involving the eigenvalue decay beyond the index $k$.

When the eigenvalues follow a power law decay as specified in Assumption~\ref{ass:power_law_spectrum:kernel}, this upper bound leads to the scaling behavior described in our main theorem. Specifically, with optimal choice of the parameter $k$, the approximation error decays as $O(m^{1-a})$, where $m$ is the sketch dimension and $a$ is the power law exponent.

This analysis demonstrates how random sketching can effectively capture the most significant components of the regression problem while incurring a controlled approximation error that diminishes as the sketch dimension increases.

\subsection{Lower Bound of Approximation Error}\label{sec:approx_multi_lower}
Having established the upper bound, we now complement our analysis by deriving a lower bound on the approximation error. This lower bound is essential for understanding the fundamental limitations of sketched regression and confirms the tightness of our scaling results.

The following lemma characterizes the lower bound on the expected approximation error, providing a guarantee on the minimum error that any sketched predictor must incur.

\begin{lemma}[Lower bound on the approximation error]\label{lem:approx_error_lower_bound_mutliregression}
    When $r(G) \geq \delta$,under Assumption~\ref{ass:parameter_prior}, the approximation error in Lemma~\ref{lem:risk_decompose} and Lemma~\ref{lem:approximation_error:multiregression} satisfies
    \begin{align*}
        \E_{W*}\mathsf{Approx} = \Omega( \sum_{i=\delta}^d \lambda_i),
    \end{align*}
    where $(\lambda_i)_{i=1}^d$ are eigenvalues of $G$ in non-increasing order.
\end{lemma}

\begin{proof}
    For any $k \leq d$, following the Eq.~\eqref{eq:inequality_approx}, we have
    \begin{align}\label{eq:approx_redefined}
         & ~ (G^{1/2} R^\top A^{-1} RG^{1/2} - I_{d\times d}) (G^{1/2} R^\top A^{-1} RG^{1/2} - I_{d\times d})^\top \notag \\
        = & ~ \begin{bmatrix}
            \alpha^2 + \beta \beta^\top & \alpha\beta + \beta \kappa \\
            \beta^\top \alpha + \kappa\beta^\top & \kappa^2 + \beta^\top \beta
        \end{bmatrix}
    \end{align}.
    Then we have
    \begin{align*}
        \E_{W^*} \mathsf{Approx} 
        = & ~ \E_{W^*} \langle \alpha^2 + \beta \beta^\top, G_{0:k}^{1/2} W_{0:k}^* W_{0:k}^* G_{0:k}^{1/2} \rangle + \E_{W^*} \langle \kappa^2 + \beta^\top \beta, G_{0:k}^{1/2} W_{k:\tau}^* W_{k:\tau}^* G_{k:\tau}^{1/2} \rangle \\
          & ~ + 2\E_{W^*} \langle \alpha \beta + \beta \kappa, G_{0:k}^{1/2} W_{0:k}^* W_{k:\tau}^* G_{k:\tau}^{1/2} \rangle\\
        = & ~ \tr[(\alpha^2 + \beta \beta^\top) G_{0:k}] + \tr[(\kappa^2 + \beta^\top \beta)G_{k:\tau}],
    \end{align*}
    where the first step follows from Eq.~\eqref{eq:approx_redefined}, and second step follows from the fact that $\E_{W^*}(W^*W^{*\top})= I_{d\times d}$

    Then, we have
    \begin{align}\label{eq:lowerbound_approx_inequality}
        \E_{W^*}\mathsf{Approx}
        = & ~ \tr[G_{0:k}^{1/2} (G_{0:k}^{1/2} + R_{0:k}^\top A_k^{-1} R_{0:k})^{-1} G_{0:k}^{-1/2} G_{0:k}] - \tr[\kappa G_{k:\tau}] \notag \\
        = & ~ \tr[(G_{0:k}^{-1} + R_{0:k}^{-1} A_{k}^{-1} R_{0:k})^{-1}] - \tr[\kappa G_{k:\tau}] \notag \\
        \geq & ~ -\tr[\kappa G_{k:\tau}]
    \end{align}
    where the first step follows from Eq.~\eqref{eq:1st_claim} and Eq.~\eqref{eq:u_vv}, the second step follows from basic matrix algebra, and the third step follows from basic inequality.    

    We define $M:=I_{(d-k)\times (d-k)} - G_{k:\tau}^{1/2} R_{k:\tau}^\top A_k^{-1} R_{k:\tau} G_{k:\tau}^{1/2}$, which is a projection matrix such taht $M^2 = M$ and $\tr[I_{(d-k)\times (d-k)}-M]=M$. Then, $M$ has $\delta$ eigenvalues $0$ and $d-k-\delta$ eigenvalues $1$. 
    
    Given that $A_k = R_{k:\tau}G_{k:\tau}R_{k:\tau}^\top$, we further have,
    \begin{align*}
        & ~ \tr[G_{k:\tau}^{1/2} (I_{(d-k)\times (d-k)} - G_{k:\tau}^{1/2} R_{k:\tau}^\top A_k^{-1} R_{k:\tau} G_{k:\tau}^{1/2}) G_{k:\tau}^{1/2}] \\
        \geq & ~ \tr[G_{k:\tau}^{1/2} [I_{(d-k)\times (d-k)} - G_{k:\tau}^{1/2} R_{k:\tau}^\top A_k^{-1} R_{k:\tau} G_{k:\tau}^{1/2}] G_{k:\tau}^{1/2}] \\
        \geq & ~ \sum_{i=1}^{d-k} u_i(I_{(d-k)\times (d-k)} - G_{k:\tau}^{1/2} R_{k:\tau}^\top A_k^{-1} R_{k:\tau} G_{k:\tau}^{1/2}) \cdot \mu_{d+1-k-i}(G_{k:\tau}) \\
        \geq & ~ \sum_{i = k + \delta}^d \lambda_i
    \end{align*}
    where the first step follows from Eq.~\eqref{eq:lowerbound_approx_inequality}, the second step follows from $A = \Omega(A_k)$ (equivalent to $-A^{-1} = \Omega(A_k^{-1})$), the third step follows from Von-Neuman's inequality, and the fourth step follows from the definition of $M$.

    For any $k \leq d$, we maximize the lower bound by setting $k = 0$.

    Thus, we complete the proof.
\end{proof}

This lower bound reveals that the approximation error is fundamentally limited by the tail sum of eigenvalues of the covariance matrix. Importantly, when combined with the power law decay assumption for eigenvalues (Assumption~\ref{ass:power_law_spectrum:kernel}), this lower bound translates to $\Omega(m^{1-a})$ where mm
m is the sketch dimension and aa
a is the power law exponent.

The matching order between the upper and lower bounds confirms that our characterization of the approximation error is tight. This establishes that the scaling behavior we observe is not merely an artifact of our analysis techniques but reflects a fundamental property of sketched multiple regression.
This result has important implications for the practice of sketched regression. It indicates that, even with optimally designed sketching algorithms, there exists an irreducible approximation error that scales according to a power law with the sketch dimension. This understanding helps practitioners make informed decisions about the tradeoff between computational resources and predictive performance.

%% file: 6_bias_error.tex
\section{Bias Error Bounds for Multiple Regression}\label{sec:bias_error}
In this section, we prove and bound the bias error for multiple regression.
In Section~\ref{sec:bias_multiple}, we prove the bias error for multiple regression.
In Section~\ref{sec:bias_multiple_upper}, we show the upper bound of bias error for multiple regression.
In Section~\ref{sec:bias_multiple_lower}, we present the lower bound result of bias error for multiple regression.

\subsection{Bias Error for Multiple Regression}\label{sec:bias_multiple}

 The bias error arises from the optimization process and measures how far our trained model remains from the optimal sketched predictor after a finite number of training iterations.
 
In this section, we provide a formal definition of the bias error for multiple regression, which will be crucial for our subsequent analysis of the overall scaling behavior.

\begin{definition}[$W^*$-dependent Bias error]\label{def:bias_variance}
    We define the $W^*$-dependent bias error as follows.
    \begin{align*}
        \mathsf{Bias}(W^*) := \| \prod_{t=1}^N (I-\gamma_t RGR^\top) V^*\|_{RGR^\top}^2
    \end{align*}
    where $V^* := (RGR^\top)^{-1}RGW^*$.
    Then, its variance error is,
    \begin{align*}
        \mathsf{Var} := \frac{\{ j : \wt{\lambda}_j \geq \frac{1}{N_{\mathrm{eff}} \gamma} \} 
        + (N_{\mathrm{eff}} \gamma)^2 \sum_{ \wt{\lambda}_j < \frac{1}{N_{\mathrm{eff}} \gamma}} \wt{\lambda}_j^2}
        {N_{\mathrm{eff}}}
    \end{align*}
    where $N_{\mathrm{eff}} := N/\log (N)$ and $(\wt{\lambda_j})_{j=1}^M$ are eigenvalues of $RGR^\top$.
\end{definition}

This definition characterizes the bias error in terms of how the optimization dynamics affect the convergence to the optimal sketched solution. The product term $\prod_{t=1}^N (I-\gamma_t RGR^\top)$ represents the effect of running SGD for N iterations with varying step sizes. This product, when applied to the optimal sketched parameter $V^*$, measures the remaining distance to convergence.

The variance term complements the bias and reflects how the stochasticity in the optimization process affects the final solution. It depends on the effective sample size $N_{\mathrm{eff}}$ and the eigenvalue spectrum of the sketched covariance matrix $RGR^\top$. The formula separates the contribution from large eigenvalues (those above the threshold $\frac{1}{N_{\mathrm{eff}} \gamma}$ and small eigenvalues, highlighting the different convergence behaviors across the eigenspectrum.

Together, these terms play a crucial role in the overall generalization performance, as they capture the optimization-related components of the error that persist even after the approximation error has been accounted for.

\subsection{Upper Bound of the Bias Term}\label{sec:bias_multiple_upper}

Having defined the bias error, we now establish an upper bound on this component to better understand its contribution to the overall generalization performance. This analysis allows us to quantify how quickly the optimization process converges to the optimal sketched solution.

The following lemma provides two different upper bounds on the bias term, offering complementary insights into its behavior under different parameter regimes.

\begin{lemma}[Upper bound on the bias term]\label{lem:bias_error_upper_bound_mutliregression}
    Suppose the initial stepsize $\gamma \leq \frac{1}{c \tr[R G R^\top]}$ for some constant $c > 1$. Then for any $W^* \in \R^{d\times p} $ and $k \in [d]$ such that $r(G) \geq k + M$, the bias term in Definition~\ref{def:bias_variance} satisfies
    \begin{align*}
        \mathsf{Bias}(W^*) = O( \frac{1}{N_{\rm eff} \gamma} \|V^*\|_F^2).
    \end{align*}
    Moreover, for any $k \leq M/3$ such that $r(G) \geq k + M$, the bias term satisfies
    \begin{align*}
        \mathsf{Bias}(W^*) = O ( \frac{\|W_{0:k}^*\|_F^2}{N_{\rm eff} \gamma} \cdot (\frac{\mu_{M/2} (A_k)}{\mu_{M}(A_k)})^2 + \|W_{k:\tau}^*\|_{G_{k:\tau}}^2 )
    \end{align*}
    with probability $1- e^{- \Omega(M)}$, where $A_k := R_{k:\tau} G_{k:\tau} R_{k:\tau}^\top$, $\{\mu(A_k) \}_{i=1}^M $denote the eigenvalues of $A_k$ in non-increasing order for some constant $c > 1$.

\end{lemma}

\begin{proof}
    Similar to Lemma~\ref{lem:approx_error_upper_bound_mutliregression} proof, we assume the covariance matrix $G = \diag\{\lambda_1, \lambda_2, ..., \lambda_d\}$ where $\lambda_i \geq \lambda_j$ for any $i \geq j$. Let $R G^{1/2} = \wt{\mathsf{\alpha}}(\wt{\Lambda}^{1/2}0) \wt{\mathsf{\beta}}^\top$ be the singular value decomposition of $R G R^\top$, where $\wt{\Lambda} := \diag\{\lambda_1, \lambda_2, ..., \lambda_d\}$ is a diagonal matrix diagonal entries is non-increasing order. Define $A_k := R_{k:\tau} G_{k:\tau} R_{k:\tau}^\top$. Then if follows from similar arguments as in Lemma~\ref{lem:approx_error_upper_bound_mutliregression} that $A_k$ is invertible.
    
    Since
    \begin{align*}
        \| \gamma_t R G R^\top \|_F = \gamma_t \wt{\lambda}_1 \leq \gamma \wt{\lambda}_1 \leq \frac{\wt{\lambda}_1}{c \sum_{i=1}^M \wt{\lambda_i}} \leq 1 
    \end{align*}

    for some constant $c > 1$ by the stepsize assumption, it follows that 
    $I_M - \gamma R G R^\top \succ 0_M$ for all $t \in [N]$. Therefore, it can be verified that
    \begin{align*}
        \prod_{t=1}^N(I_M - \gamma_t R h R^\top) R G R^\top \prod_{t=1}^N((I_M - \gamma_t R h R^\top) \preceq (I_M - \gamma R h R^\top)^{N_{\rm eff}} R G R^\top((I_M - \gamma R h R^\top)^{N_{\rm eff}} =: \mathsf{M},
    \end{align*}
    and by definition of $\mathsf{Bias}(W^*)$ in Definition~\ref{def:bias_variance}, we have
    \begin{align}
        \mathsf{Bias}(W^*) = & ~ \Theta(\|\prod_{t=1}^N(I - \gamma_t R G R^\top)V^*\|_{R G R^\top}^2) \notag \\ 
        \leq & ~ \Theta(\|(I - \gamma R G R^\top)^{N_{\rm eff}}V^*\|_{R G R^\top}^2) \notag \\
        = & ~ \Theta (\langle\mathsf{M}, b^*b^{*\top}\rangle).
    \end{align}
    where the first step follows from Definition~\ref{def:bias_variance}, the second step follows from basic matrix algebra, and the third step follows from the definition of $\mathsf{M}$.

    Note that the eigenvalues of $\mathsf{M}$ are $\{\wt{\lambda_i}(1 - \gamma \wt{\lambda_i})^{2N_{\rm eff}}\}_{i=1}^M$. Since the function $f(x) = x(1 - \gamma x)^{2N_{\rm eff}}$ is maximized at $x_o = 1/[(2N_{\rm eff} + 1)\gamma]$ for $x \in [0,1/\gamma]$ with $f(x) = O( 1/(N_{\rm eff}\gamma))$, it follows that
    \begin{align}\label{eq:M_norm}
        \| \mathsf{M} \|_2 \leq c / (N_{eff} \gamma)
    \end{align}
    for some constant  $c > 0$. The first part of the lemma follows immediately.

    Now we prove the second part of the lemma. Recall that $V^* = (R G R^\top)^{-1} R G W^*$. Substituting
    \begin{align*}
        R G = (R_{0:k} G_{0:k}~~R_{k:\tau}G_{k:\tau})
    \end{align*}
    into $V^*$, we obtain
    \begin{align*}
        \langle \mathsf{M}, V^*V^{*\top} \rangle 
        = & ~ \langle M, ((R G R^\top)^{-1}R G W^*)((R G R^\top)^{-1}R G W^*)^\top \rangle \\
        = & ~ W^{* \top} G R^\top (R G R^\top)^{-1}\mathsf{M}(R G R^\top)^{-1} R G W^* \\
        \leq & ~ 2T_1 + 2T_2
    \end{align*}
    where
    \begin{align*}
        T_1 := & ~(W_{0:k}^*) G_{0:k} R_{0:k}^\top (R G R^\top)^{-1} \mathsf{M} (R G R^\top)^{-1} R_{0:k} G_{0:k} W_{0:k}^* \\
        T_2 := & ~(W_{k:\tau}^*) G_{k:\tau} R_{k:\tau}^\top (R G R^\top)^{-1} \mathsf{M} (R G R^\top)^{-1} R_{k:\tau} G_{k:\tau} W_{k:\tau}^*.
    \end{align*}
     By definition of $T_1$, we have
    \begin{align*}
        T_1 \leq \|G_{0:k} R_{0:k}^\top (R G R^\top)^{-1} \mathsf{M} (R G R^\top)^{-1} R_{0:k} G_{0:k} \|_2 \cdot \|W_{0:k}^*\|_F^2 
    \end{align*}
    Then we show,
    \begin{align*}
        & \| G_{0:k} R_{0:k}^\top (R G R^\top)^{-1} \mathsf{M} (R G R^\top)^{-1} R_{0:k} G_{0:k} \|_F \\
        \leq & ~ \| \mathsf{M}_2 \| \cdot (R G R^\top) ^{-1} R_{0:k} G_{0:k} \|_F^2 \\
        \leq & ~ \frac{c}{N_{\rm eff} \gamma} \|(R G R^\top)^{-1} R_{0:k} G_{0:k} \|_F^2
    \end{align*}
    for some constant $c > 0$,
    where the first step follows from the first part of $T_1$ inequality, and the second step follows from Eq.~\eqref{eq:M_norm}.

    It remains to show
    \begin{align}
        \| (R G R^\top)^{-1} R_{0:k}G_{0:k} \|_F \leq c \cdot \frac{\mu_{M/2}(A_k)}{\mu_M(A_k)}
    \end{align} 
    for some constant $c > 0$ with probability $1 - e^{\Omega(M)}$. Since $R G R^\top = R_{0:k} G_{0:k} R_{0:k}^\top + A_k$, we have
    \begin{align}
        (R G R^\top)^{-1} R_{0:k} G_{0:k}
        = & ~ (A_k^{-1} - A_k^{-1} R_{0:k}(G_{0:k}^{-1} + R_{0:k}^\top A_k^{-1} R_{0:k})^{-1} R_{0:k}^\top A_k^{-1}) R_{0:k} G_{0:k} \notag \\
        = & ~ A_k^{-1} R_{0:k} G_{0:k} - A_{k}^{-1} R_{0:k}(G_{0:k}^{-1} + R_{0:k}^\top A_k^{-1} R_{0:k})^{-1} R_{0:k}^\top A_k^{-1} R_{0:k} h_{0:k} \notag \\
        = & ~ A_k^{-1} R_{0:k} (G_{0:k}^{-1} + R_{0:k}^\top A_k^{-1} R_{0:k})^{-1} G_{0:k}^{-1} G_{0:k} \notag \\
        = & ~ A_k^{-1} R_{0:k} (G_{0:k}^{-1} + R_{0:k}^\top A_k^{-1} R_{0:k})^{-1}, 
    \end{align}
    where the first step follows from $R G R^\top = R_{0:k} G_{0:k} R_{0:k}^\top + A_k$, the second step follows from using Woodbury's identity, the third step follows from basic matrix algebra, and the fourth step follows from basic matrix algebra.

    Since
    \begin{align*}
        G_{0:k}^{-1} + R_{0:k}^\top A_k^{-1} R_{0:k} \succeq R_{0:k}^\top A_k^{-1} R_{0:k},
    \end{align*}
    if follows that 
    \begin{align*}
        \| (G_{0:k}^{-1} + R_{0:k}^\top A_k^{-1}R_{0:k})^{-1} \|_F \leq \| (R_{0:k}^\top A_k^{-1} R_{0:k})^{-1} \|_F. 
    \end{align*}
    Therefore, with probability at least $1 - e^{-\Omega(M)}$
    \begin{align*}
        \| A_k^{-1} R_{0:k} (G_{0:k}^{-1} + R_{0:k}^\top A_k^{-1} R_{0:k})^{-1} \|_F 
        \leq & ~ \| A_k^{-1} \|_F \cdot \|R_{0:k}\|_2 \cdot \|(G_{0:k}^{-1} + R_{0:k}^\top A_k^{-1} R_{0:k})^{-1} \|_F \\
        \leq & ~ \| A_k^{-1} \|_F \cdot \|R_{0:k}\|_2 \cdot \| (R_{0:k}^\top A_k^{-1} R_{0:k})^{-1} \|_F \\
        \leq & ~ \frac{\|A_k^{-1}\|_F \cdot \|R_{0:k}\|_2}{\mu_{min}(R_{0:k}^\top A_k^{-1} R_{0:k})} \\
        = & ~ O(\frac{\|A_k^{-1}\|_F}{\mu_{min}(R_{0:k}^\top A_k^{-1} R_{0:k})})
    \end{align*}
    where the first step follows from the property of norms, the second step follows from basic inequalities, 
    the third step follows from the definition of $\mu_{\min}$, and the fourth step follows from the fact that 
    $\|R_{0:k}\|_2 = \sqrt{\| R_{0:k}^{\top} R_{0:k} \|_F} \leq c$ for some constant $c > 0$ 
    when $k \leq M/2$ with probability at least $1 - e^{-\Omega(M)}$.

    Since $R_{0:k}$ is independent of $A_k$ and the distribution of $R_{0:k}$ is rotationally invariant, we may write $R_{0:k}^\top A_k^{-1} R_{0:k} = \sum_{i=1}^M \frac{1}{\wt{\lambda}_{M-i}} \wt{s}_i \wt{s}_i^\top$, where $\wt{s}_i \sim \mathcal{N}(0,I_{k\times k}/M)$ and $(\wt{\lambda}_i)_{i=1}^M$ are eigenvalues of $A_k$ in non-increasing order. Therefore, for $k \leq M/3$
    \begin{align}
        R_{0:k}^\top A_k^{-1} R_{0:k} = \sum_{i=1}^M \frac{1}{\wh{\lambda}_{M-i}} \wt{s}_i \wt{s}_i^\top \succeq \sum_{i=1}^{M/2} \frac{1}{\wh{\lambda}_{M-i}} \wt{s}_i \wt{s}_i^\top \succeq \frac{1}{\wh{\lambda}_{M/2}} \sum_{i=1}^{M/2} \wt{s}_i \wt{s}_i^\top \succeq \frac{c I_k}{\wh{\lambda}_{M/2}}
    \end{align}
    for some constant $c > 0$ with probability at least $1 - e^{-\Omega(M)}$, where in the last line we again use the concentration properties of Gaussian covariance matrices(see Theorem 6.1 in~\cite{w19}). As a direct consequence, we have
    \begin{align*}
        \| A_k^{-1} R_{0:k} (G_{0:k}^{-1} + R_{0:k}^\top A_{k}^{-1} R_{0:k})^{-1} \|_F  \leq c \cdot \frac{\mu_{M/2}(A_k)}{\mu_{M} (A_k)}
    \end{align*}
    with probability at least $1 - e^{-\Omega(M)}$ for some constant $c > 0$. This concludes the proof.

    By definition of $T_2$, we have
    \begin{align*}
        T_2 
        = & ~ {W_{k:\tau}^*}^\top G_{k:\tau} R_{k:\tau}^\top (R G R^\top)^{-1/2} (I_M - \gamma R G R^\top)^{2 N_{\rm eff}} (R G R^\top)^{-1/2} R_{k:\tau} G_{k:\tau} W_{k:\tau}^* \\
        \leq & ~ {W_{k:\tau}^*}^\top G_{k:\tau} R_{k:\tau}^\top (R G R^\top)^{-1} R_{k:\tau} G_{k:\tau} W_{k:\tau}^* \\
        \leq & ~ \| G_{k:\tau}^* R_{k:\tau}^\top (R G R^\top)^{-1} R_{k:\tau} G_{k:\tau}^{1/2} \|_F \cdot \|W_{k:\tau}^*\|_{G_{k:\tau}}^2 \\
        \leq & ~ \|W_{k:\tau}^*\|_{G_{k:\tau}}^2,
    \end{align*}
    where the first step follows from the definition of $T_2$, the second step follows from basic matrix algebra, the third step follows from the definition of norms, and the fourth step follows from the following.
    \begin{align*}
        \| G_{k:\tau}^{1/2} R_{k:\tau}^\top (R G R^\top)^{-1} R_{k:\tau} G_{k:\tau}^{1/2} \|_2 
        = & ~ \| G_{k:\tau}^{1/2} R_{k:\tau}^\top (R_{0:k} G_{0:k} R_{0:k}^\top + R_{k:\tau} G_{k:\tau} R_{k:\tau}^\top)^{-1} R_{k:\tau} G_{k:\tau}^{1/2} \|_2 \\
        \leq & ~ \| G_{k:\tau}^{1/2}R_{k:\tau}^\top A_k^{-1} R_{k:\tau} G_{k:\tau}^{1/2} \|_2 \leq 1.
    \end{align*}
    Thus, we complete the proof.
\end{proof}

The first bound provides a simple characterization that relates the bias to the norm of the optimal sketched parameter and the effective sample size. This shows that with more training examples (larger $N_{\rm eff}$) or a larger step size (within stability limits), the bias decreases inversely proportionally.

The second bound offers a more refined analysis by separating the contributions from different parameter components. It shows that the bias for the head parameters (indexed from $0$ to $k$) depends on their magnitude scaled by a factor involving the condition number of the sketched covariance matrix. Meanwhile, the bias for the tail parameters (indexed from $k$ to $\tau$) depends directly on their energy under the covariance metric.
This refined bound is particularly useful when combined with our power law assumption on eigenvalues, as it allows us to derive the scaling behavior described in our main theorem. Specifically, when optimizing the choice of k, this analysis leads to a bias term that scales as $O(\frac{1}{(N_{\rm eff}\gamma)^{(a-1)/a}})$, where a is the power law exponent.

These upper bounds on the bias term, together with our previous results on approximation error, provide a comprehensive understanding of the generalization performance of sketched multiple regression under SGD training.

\subsection{Lower Bound of the Bias Term}\label{sec:bias_multiple_lower}
To complement our upper bound analysis, we now establish a lower bound on the bias error. This lower bound is crucial for demonstrating the tightness of our scaling results and understanding the fundamental limitations of optimization in sketched multiple regression.
The following lemma provides a characterization of the expected bias error under a prior distribution on the parameters:

\begin{lemma}[Lower bound on the bias term]\label{lem:bias_error_lower_bound_mutliregression}

    Suppose $W^*$ follows some prior distribution and the initial stepsize $\gamma \leq \frac{1}{c \tr[R G R^\top]}$ for some constant $c > 2$. Let $G^W := \E W^* W^{*\top}$. Then the bias term in Definition~\ref{def:bias_variance} satisfies
    \begin{align*}
        \E_{W^*} \mathsf{Bias}(W^*) = \Omega( \sum_{i: \wt{\lambda}_i< 1/(\gamma N_{\rm eff}) } \frac{\mu_i(R G G^W G R^\top)}{\mu_i(R G R^\top)})
    \end{align*}
    almost surely, where $\mathsf{M}_N := R G R^\top(I - 2\gamma R G R^\top)^{2 N_{\rm eff}}$.
    
\end{lemma}

\begin{proof}

    Adopt the notations in the proof of Lemma~\ref{lem:bias_error_upper_bound_mutliregression} By definition of the bias term, we have
    \begin{align}
        \mathsf{Bias}(W^*) 
        \eqsim & ~ \| \prod_{t=1}^N (I - \gamma_t R G R^\top) V^* \|_{R G R^\top}^2 \notag \\
        = & ~ \langle R G R^\top \prod_{t=1}^N(I - \gamma_t R G R^\top)^{2 N_{\rm eff}}, V^*V^{* \top} \rangle \notag \\
        \geq & ~ \langle R G R^\top (I - \sum_{t=1}^N \gamma_t R G R^\top)^{2 N_{\rm eff}}, V^*V^{*\top} \rangle \notag \\
        \geq & ~ \langle R G R^\top (I - 2\gamma R G R^\top)^{2 N_{\rm eff}}, V^*V^{*\top} \rangle =: \langle \mathsf{M}_N, V^*V^{*\top} \rangle,
    \end{align}
    where the first step follows from Definition~\ref{def:bias_variance}, the second step follows from the definition of norms, the third step follows from $I_M - 2\gamma_t R G R^\top \succ 0_M$ for all $t \in [N]$ established in the proof of Lemma~\ref{lem:bias_error_upper_bound_mutliregression}, $\sum_{i=1}^N \gamma_i \leq 2 \gamma N_{\rm eff}$, and the fact that $(1 - w)(1 - v) \geq 1 - w - v$ for $w, v > 0$, and the fourth step follows from basic algebra.

     Substituting the definition of $V^*$ in Definition~\ref{def:bias_variance} into the expression, we obtain 
    \begin{align*}
        \E_{W^*}\mathsf{Bias}(W^*) 
        \gtrsim & ~ \E_{W^*} \langle \mathsf{M}_N, V^*V^{*\top} \rangle = \E_{W^*} \langle \mathsf{M}_N, ((R G R^\top)^{-1} R G W^*)^{\otimes 2} \rangle \\
        = & ~ \tr[(G R^\top)(R G R^\top)^{-1} \mathsf{M}_N (R G R^\top)^{-1} R G G^W] \\
        = & ~ \tr[(R G R^\top)^{-1} \mathsf{M}_N(R G R^\top)^{-1}R G G^W G R^\top]\\
        \geq & ~ \sum_{i=1}^M \mu_{M-i+1}((R G R^\top)^{-1} \mathsf{M}_N(R G R^\top)^{-1}) \cdot \mu_i(R G G^W G R^\top),
    \end{align*}
    where the last line uses Von Neumann's trace inequality. Continuing the calculation, we have 
    \begin{align*}
        \E_{W^*} \mathsf{Bias}(W^*) 
        = & ~  \Omega( \sum_{i=1}^M \frac{\mu_i(R G G^W G R^\top)}{\mu_i((R G R^\top)^2 {\mathsf{M}_N}^{-1})}) \\
        = & ~ \Omega(\sum_{i=1}^M \frac{\mu_i(R G G^W G R^\top)}{\mu_{i}((R G R^\top)(I - 2\gamma R G R^\top)^{-2N_{\rm eff}})}) \\
        = & ~ \Omega( \sum_{i:\wt{\lambda}_i < 1/(\gamma N_{\rm eff})} \frac{\mu_i(R G R^W G R^\top)}{\mu_i(R G R^\top)}),
    \end{align*}
    where the first inequality uses $\mu_{M+i-1}(A) = \mu_i^{-1}(A^{-1})$ for any positive definite matrix $A \in \R^{M \times M}$, and the second line follows from the definition of $\mathsf{M}_N$ and the fact that $(1 -\lambda \gamma N_{\rm eff})^{-2 N_{\rm eff}} \lesssim 1$ when $\lambda < (\gamma N_{\rm eff})$.
\end{proof}

This lower bound reveals that the bias error is fundamentally limited by the ratio of eigenvalues of the matrices $R G G^W G R^\top$ and $R G R^\top$, particularly for eigenvalues below the threshold $1/(\gamma N_{\rm eff})$. These smaller eigenvalues correspond to directions in the parameter space that are harder to optimize and require more iterations to converge.

When combined with our assumptions on the power law decay of eigenvalues and the parameter prior, this lower bound translates to $\Omega(\frac{1}{(N_{\rm eff}\gamma)^{(a-1)/a}})$ where $a$ is the power law exponent. This matches the order of the upper bound, confirming that our characterization of the bias term is tight.
The presence of the matrix $G^W$ in the bound highlights how the prior distribution on parameters affects the optimization difficulty. When $G^W = I$ as specified in Assumption \ref{ass:parameter_prior}, the bound simplifies and leads to our scaling law result.

This lower bound on bias, together with our previous results on approximation error, establishes the fundamental trade-offs in sketched multiple regression and provides theoretical justification for the scaling behaviors observed in practice.

%% file: 7_excess_error.tex
\section{Excess Error for Multiple Regression}\label{sec:excess}

In this section, we discussed the excess error for multiple regression. 
In Section~\ref{sec:hypercontractivity}, we proved the hypercontractivity and misspecified noise under sketching. 
In Section~\ref{sec:gaussianity_noise}, we show Gaussianity and well-specified noise under sketching.
In Section~\ref{sec:excess_error}, we analyze the excess error using the hypercontractivity, Gaussianity and well-specified under sketching.

\subsection{Hypercontractivity and the Misspecified Noise under Sketch Feature}\label{sec:hypercontractivity}

In this section, we study the hypercontractivity and misspecified noise under sketched features.

We first introduce our assumptions as follows.

\begin{assumption}[General distributional conditions.]\label{ass:general_distribution}
    Assume the following about the data distribution.
    \begin{itemize}
        \item Hypercontractivity. There exits $\alpha \geq 1$ such that for every PSD matrix $A$ it holds that
        \begin{align*}
            \E [xx^\top]Axx^\top \preceq \alpha \tr[GA]G
        \end{align*}
        \item Misspecified model. There exits $\sigma^2 > 0$ such that $\E[y-x^\top W^*]^2xx^\top \preceq \sigma^2 G$
    \end{itemize}
    
\end{assumption}

Then, we can prove the hypercontractivity.

\begin{lemma}[Hypercontractivity and the misspecified noise under sketched feature]\label{lem:hypercontractivity and misspecified noise}
    Suppose that Assumption~\ref{ass:general_distribution} hold. Conditioning on the sketch matrix $R$, for every $PSD$ matrix $A \in \R^{M \times M}$, we have 
    \begin{align*}
        \E[(R x)(R x)^\top] A(R x)(R x)^\top \preceq \alpha \tr[R G R^\top A] R G R^\top.
    \end{align*}
    Moreover, for the minimizer of $\loss_R(V)$, that is, $V^*$ defined in Lemma~\ref{lem:approximation_error:multiregression}, we have
    \begin{align*}
          \E[\|(V^{*\top} R)x - y\|_2^2] (R x)(Rx)^{\top} \preceq 2(\sigma^2 + \alpha \|W^*\|_G^2)R G R^\top. 
    \end{align*}
    The expectation in the above is over $(x,y)$.
    
\end{lemma}

\begin{proof}
    The first part is a direct application of Assumption~\ref{ass:general_distribution}
    \begin{align*}
        \E[(R x)(R x)^\top] A(R x)(R x)^\top
        = & ~ R(\E [x x^\top](R^\top A R)x x^\top)R^\top \\
        \preceq & ~ R(\alpha \tr[G R^\top A R]G)R^\top\\
        = & ~ \alpha \tr[R G R^\top A]R G R^\top.
    \end{align*}
    For the second part, we first show that
    \begin{align*}
        \E[\|(V^{*\top} R)x - y\|_2^2] xx^\top 
        \preceq & ~ 2 \E[\| W^{*\top} x - y \|_2^2] xx^\top + 2\E [\langle (W^{*\top} -V^{*\top}R)^\top, x\rangle^2] xx^\top 
        \\
        \preceq & ~ 2 \sigma^2G + 2\alpha \tr[G (W^{*\top} -V^{*\top}R)^\top (W^{*\top} -V^{*\top}R) ] G
        \\
        \preceq & ~ 2 \sigma^2G + 2\alpha \| G (W^{*\top} -V^{*\top}R)^\top (W^{*\top} -V^{*\top}R)\|_F G,
    \end{align*}
    where the second inequality is by Assumptions~\ref{ass:general_distribution} and the last inequality is by trace smaller than Frobenius norm. 
    
    From the proof of Lemma~\ref{lem:approximation_error:multiregression}, we know that
    \begin{align*}
        \| G (W^{*\top} -V^{*\top}R)^{\top}(W^{*\top} -V^{*\top}R) \|_F = \mathsf{Approx} \leq \|W^*\|_G^2, ~~\mathrm{almost~surely}.
    \end{align*}
    So we have
    \begin{align*}
        \E[\|(V^{*\top} R)x - y\|_2^2] xx^\top\preceq 2 (\sigma^2 + \alpha \|W^*\|_G^2)G.
    \end{align*}
    Left and right multiplying both sides with $R$ and $R^\top$, we obtain the second claim.
    
\end{proof}

\subsection{Gaussianity and Well-specified Noise under Sketched Features}\label{sec:gaussianity_noise}

In this section, we prove Gaussianity and well-specified noise under sketching.

\begin{lemma}[Gaussianity and well-specified noise under sketched features]\label{lem:gaussianity and noise}
    Suppose that Assumptions~\ref{ass:gaussian_feature} and Assumption~\ref{ass:well_specified} hold.Conditional on the sketch matrix $R$, we have
    \begin{align*}
        R x \sim \mathcal{N}(0, R G R^\top).
    \end{align*}
    Moreover, for the minimizer of $\loss_R (V)$, that is, $V^*$ defined in Lemma~\ref{lem:approximation_error:multiregression}, we have
    \begin{align*}
        \E[y|R x] =  R xV^{*},~\E[\|y - V^{*\top} R x\|_2^2] = \sigma^2 + \mathsf{Approx} \geq \sigma^2.
    \end{align*}
\end{lemma}

\begin{proof}
    The first claim is a direct consequence of Assumption~\ref{ass:gaussian_feature}.
    For the second claim, by Assumption~\ref{ass:gaussian_feature} and Lemma~\ref{lem:approximation_error:multiregression}, we have
    \begin{align*}
        \E[y|x] 
        = & ~ W^{*\top}x\\
        = & ~ V^{*\top}Rx +(W^{*\top}-V^{*\top}R)x \\
        = & ~  V^{*\top}Rx + (I - (R G R^\top)^{-1} R G W^*)x \\
        = & ~   \underbrace{V^{*\top}}_{p\times m} \underbrace{R}_{m\times d} G^{\frac{1}{2}} G^{-\frac{1}{2}} \underbrace{x}_{d\times 1} + (I - G^{\frac{1}{2}} R^\top (R G R^\top)^{-1} R G^{\frac{1}{2}} \underbrace{W^*}_{d\times p})^\top R G^{\frac{1}{2}} G^{-\frac{1}{2}}x 
    \end{align*}
    Notice that 
    \begin{align*}
        G^{\frac{1}{2}} x \sim \mathcal{N}(0,I),
    \end{align*}
    by Assumption~\ref{ass:gaussian_feature} and that 
    \begin{align*}
        R G^{\frac{1}{2}} [I - G^{\frac{1}{2}} R^\top (R G R^\top)^{-1} R G^{\frac{1}{2}}] = 0, 
    \end{align*}
    therefore $R x = R G^{\frac{1}{2}} G^{-\frac{1}{2}} $ $x$ is independent of $ (I - (R G R^\top)^{-1} R G  G^{-\frac{1}{2}}\underbrace{W^*}_{d\times p})^\top RG^{-\frac{1}{2}}x.$ 

    Taking expectation over the second random vector, we find
    \begin{align*}
        \E[y|R x] = \E[\E[y|x]] = \| R G^{\frac{1}{2}} G^{-\frac{1}{2}} x V^{*} \|_F = \| R x V^* \|_F.
    \end{align*}
    It remains to show
    \begin{align*}
        \E[\|y - V^{*\top} R x\|_2^2] = \sigma^2 + \mathsf{Approx}.
    \end{align*}
    Then, we have,
    \begin{align*}
         \E[\|y - V^{*\top} R x\|_2^2]
        = & ~ \loss(V^{*\top} R) \\
        = & ~ \mathsf{Approx} + \loss(W^*) \\
        = & ~ \mathsf{Approx} + \sigma^2 \\
        \geq & ~ \sigma^2,
    \end{align*}
    where the first step follows from the proof of Lemma~\ref{lem:approximation_error:multiregression}, the second equality is by the definition of $\mathsf{Approx}$, and the third equality is by Assumption~\ref{ass:well_specified}. We have completed the proof.
    
\end{proof}

\subsection{Excess Error for Multiple Regression}\label{sec:excess_error}
In this section, we present the proof for the excess error bounds for multiple regression with Lemma~\ref{lem:hypercontractivity and misspecified noise} and~\ref{lem:gaussianity and noise}. 

\begin{lemma}[Excess error bounds]\label{lem:excee_risk_bounds}
    Consider the excess error in Lemma~\ref{lem:risk_decompose} induced by the SGD from Lemma~\ref{lem:multi_sgd}. Let
    \begin{align*}
        N_{\mathrm{eff}} := N / \log N,~{\sf R N R} := (\| W^* \|_G^2 + \|V_0\|_{R G R^\top}^2)/\sigma^2.
    \end{align*}
    Then conditioning on the sketch matrix $R$, for any $W^* \in \R^{d\times p}$

    1.Under Assumption~\ref{ass:general_distribution}, we have 
    \begin{align*}
        \E [\mathsf{Excess}] = O ( \| \prod_{t=1}^N(I - \gamma_t R G R^\top)(V_0 - V^*) \|_{R G R^\top}^2 + (1 + \alpha R N R)\sigma^2 \cdot \frac{D_{\mathrm{eff}}}{N_{\mathrm{eff}}})
    \end{align*}
    2.Under Assumptions~\ref{ass:gaussian_feature} and~\ref{ass:well_specified}, we have
    \begin{align*}
        \E [\mathsf{Excess}] = \Omega( \| \prod_{t=1}^N(I - \gamma_t R G R^\top)(V_0 - V^*) \|_{R G R^\top}^2 + \sigma^2 \cdot \frac{D_{\mathrm{eff}}}{N_{\mathrm{eff}}})
    \end{align*}
    when $\gamma = O( \frac{1}{c \tr[R G R^\top]})$ for some constant $c > 0$
    In both results, the expectation is over$(x_t, y_t)_{t=1}^N$, and
    \begin{align*}
        D_{\mathrm{eff}} := \# \{ \wt{\lambda}_j \geq 1 / (N_{\mathrm{eff}} \gamma)^2 \sum_{\wt{\lambda}_j < 1 / (N_{\mathrm{eff}} \gamma )} \wt{\lambda}_j^2,
    \end{align*}
    where $\wt{\lambda}_{j=1}^M$ are eigenvalue of $R G R^\top$.

\end{lemma}

\begin{proof}
    For the upper bound, we can have the following from Lemma~\ref{lem:hypercontractivity and misspecified noise} 
    \begin{align*}
        \wt{\sigma}^2 = 2(\sigma^2 + \alpha\|W^*\|_G^2).
    \end{align*}
    Using the Corollary 3.4 in~\cite{wzb+22}, we can apply the upper bound (setting their index set $\mathbb{K} = \emptyset)$ to get 
    \begin{align*}
        \E [\mathsf{Excess}] = O( \| \prod_{t=1}^N(I - \gamma_t R G R^\top)(V_0 - V^*) \|_{R G R^\top}^2 + (\| V^* - V_0 \|_{R G R^\top}^2 + \wh{\sigma}^2) \frac{D_{\mathrm{eff}}}{N_{\mathrm{eff}}}).
    \end{align*}
    We verify that 
    \begin{align*}
        \| V^* - V_0 \|_{R G R^\top}^2
        \leq & ~ 2 \|G^{\frac{1}{2}} R^\top V^* \|_F^2 + 2 \|V_0\|_{R G R^\top}^2 \\
        = & ~ 2 \|G^{\frac{1}{2}}R^\top(R G R^\top)^{-1}R G W^*\|_F^2 + 2\|V_0\|_{R G R^\top}^2 \\
        \leq & ~ 2 \|G^{\frac{1}{2}} W^*\|_F^2 + 2\|V_0\|_{R G R^\top}^2 \\
        = & ~ 2 \|W^*\|_G^2 + 2\|V_0\|_{R G R^\top}^2,
    \end{align*}
    where the first step follows from the definition of $V^*$ and $V^0$, the second step follows from the definition of $V^*$ (see Lemma~\ref{lem:approximation_error:multiregression}), the third step follows from the definition of $W^*$, and the last step follows from the definition of norms.
    
    which implies that
    \begin{align*}
        (\|V^* - V_0\|_{R G R^\top}^2 + \wt{\sigma}^2) 
        \leq & ~ 2\|W^*\|_G^2 + 2\|V_0\|_{R G R^\top}^2 + 2(\sigma^2 + \alpha\|W^*\|_G^2) \\
        = & ~ O((1 + \alpha {\sf SNR}) \sigma^2).
    \end{align*}
    Substituting, we get the upper bound.

    For the lower bound, Lemma~\ref{lem:gaussianity and noise} shows $Rx$ is Gaussian with $\beta = 1$. Besides, Lemma~\ref{lem:gaussianity and noise} shows that the linear regression problem is well-specified, with the noise level being
    \begin{align*}
        \wt{\sigma}^2 = \sigma^2 + \mathsf{Approx} \geq \sigma^2
    \end{align*}
    Although the lower bound in Corollary 3.4 in~\cite{wzb+22} is stated for Gaussian additive noise (see their Assumption 2'), it is easy to check that the lower bound holds for any well-specified noise as described by Lemma~\ref{lem:gaussianity and noise}. Using the lower bound in Corollary 3.4 from~\cite{wzb+22}, we obtain
    \begin{align*}
        \E [\mathsf{Excess}] = \Omega( \| \prod_{t=1}^N(I - \gamma_t R G R^\top)(V_0 - V^*) \|_{R G R^\top}^2 + \wt{\sigma}^2 \frac{D_{\mathrm{eff}}}{N_{\mathrm{eff}}}).
    \end{align*}
    Plugging in $\wt{\sigma}^2 \geq \sigma^2$ gives the desired lower bound.
\end{proof}

%% file: 8_conclusion.tex
\section{Conclusion}\label{sec:conclusion}

In this paper, we presented a detailed scaling law analysis of multiple regression and kernel regression in a sketched predictor setting, focusing on the theoretical aspects of generalization error bounds. By extending classical linear regression assumptions to multiple regression and kernel regression settings, we derived meaningful bounds on the excess error, bias, and variance under specific conditions involving data covariance spectrum and step size constraints for stochastic gradient descent (SGD). These findings contribute to a deeper understanding of how sketching techniques and kernel-based methods can be applied in large-scale learning problems without compromising generalization performance. Moreover, we formulated the scaling law for both multiple and kernel regression from the generalization error bounds.

Future work may focus on extending the scaling law to other nonlinear models and further exploring the impact of different sketching techniques and hyperparameter settings on generalization error. In addition, future work could explore incorporating advanced kernel methods, such as deep kernel learning or neural tangent kernels (NTK), to understand the scaling law in these more complex, highly nonlinear settings. Investigating how these techniques scale with the size of the dataset and the complexity of the kernel could provide insights into their suitability for practical large-scale applications.

%% file: 9_app_prelim.tex
\paragraph{Roadmap} 
In Section~\ref{sec:app_prelim}, we present more notations, facts, and statements related to our work. 
In Section~\ref{sec:app_related_work}, we present more related work.

\section{Preliminaries}\label{sec:app_prelim}
In this section, we introduce some statements to better understand our work. 
In Section~\ref{sec:facts}, we introduce some basic facts.

\subsection{Baisc PSD}\label{sec:facts}
\begin{fact}
     Let $u, v \in \R^n$, we have:
\begin{itemize}
    \item $uu^\top \preceq \|u\|_2^2 \cdot I_n$.
    \item $\diag(u) \preceq \|u\|_2 \cdot I_n$.
    \item $\operatorname{diag}(u \circ u) \preceq \|u\|_2^2 \cdot I_n$.
    \item $uv^\top + vu^\top \preceq uu^\top + vv^\top$.
    \item $uv^\top + vu^\top \succeq -(uu^\top + vv^\top)$.
    \item $(v \circ u)(v \circ u)^\top \preceq \|v\|_2^2 \|u\|_\infty^2 uu^\top$.
\end{itemize}
\end{fact}

%% file: 10_app_related_work.tex
\section{More Related Work}\label{sec:app_related_work}

\paragraph{Large Language Models.}
Transformer-based architectures have rapidly become the dominant method for tackling a wide range of natural language processing (NLP) tasks, thanks to their scalability, flexibility, and ability to capture complex linguistic patterns~\cite{vsp+17}. As these models are scaled up to billions, or even trillions, of parameters and trained on vast, diverse datasets, they are referred to as large language models (LLMs) or foundation models~\cite{bha+21}. LLMs are designed to generalize effectively across numerous downstream tasks, displaying impressive adaptability and performance. Examples of these models include prominent architectures like BERT, which revolutionized contextual language understanding~\cite{dclt19}, PaLM, which excels at multilingual and multitask capabilities~\cite{cnd+23}, Llama, which is optimized for more efficient deployment in research and industry~\cite{tli+23}, as well as widely-used systems such as ChatGPT and GPT-4, both of which have pushed the boundaries of conversational AI~\cite{chatgpt,gpt4}. These large-scale models have demonstrated remarkable generalization across a broad array of downstream applications, ranging from machine translation and question-answering to summarization, text generation, and more complex reasoning tasks~\cite{bce+23}. As LLMs continue to evolve, a variety of techniques have emerged to enhance their adaptability and specialize them for specific use cases. One common approach involves adding adapter modules, which allow fine-tuning on new tasks without modifying the entire model~\cite{eyp+22, zhz+23, ghz+23}. 

Calibration techniques have also been introduced to adjust model predictions to ensure better reliability across diverse inputs and settings~\cite{zwf+21, cpp+23}. To further increase the efficacy of LLMs in real-world applications, multitask fine-tuning has gained prominence, enabling models to be trained on a variety of related tasks simultaneously, enhancing their performance across these domains~\cite{gfc+21b, vnr+23}. Complementary to this are prompt-tuning methods, where the input prompt is carefully engineered to guide the model's behavior, allowing adaptation without extensive retraining~\cite{gfc+21b, lac+21}.

\paragraph{Theoretical Machine Learning.}
Our work also takes inspiration from the following Machine Learning Theory work. Some works analyze the expressiveness of a neural network using the theory of circuit complexity~\cite{lls+25_gnn,kll+25_var_tc0,lls+24_rope_tensor_tc0,cll+24_mamba,cll+24_rope,gkl+25_flowar_tc0}. Some works optimize the algorithms that can accelerate the training of a neural network~\cite{llsz24,klsz24,dlms24,dswy22_coreset,haochen3,haochen4,dms23_spar,cll+25_deskreject,sy23,swyy23,lss+22,lsx+22,hst+22,hsw+22,hst+20,bsy23,dsy23,syyz23_weighted,gsy23_coin,gsy23_hyper,gsyz23,gswy23,syzz24,lsw+24,lsxy24,hsk+24,hlsl24,cll+25_universal_approximator,ccl+25}. Some works analyze neural networks via regressions~\cite{cll+24_icl,gms23,lsz23_exp,gsx23,ssz23_tradeoff,css+23,syyz23_ellinf,syz23,swy23,syz23_quantum,lls+25_grok}. Some works use reinforcement learning to optimize the neural networks~\cite{haochen1,haochen2,yunfan1,yunfan2,yunfan3,yunfan4,lswy23}. Some works optimize the attention mechanisms~\cite{sxy23,lls+24_conv}.

\paragraph{Accelerating Attention Mechanisms.}
The attention mechanism faces significant computational challenges due to its quadratic complexity with respect to context length as sequence lengths increase in modern large language models~\cite{gpto1,llama3_blog,claude3_pdf}. To overcome this bottleneck, researchers have developed polynomial kernel approximation techniques \citep{aa22} that utilize low-rank approximations to efficiently compute the attention matrix. These approaches substantially improve computational efficiency, enabling attention layers to operate with nearly linear time complexity during both training and inference \citep{as23, as24b}. Such methods have been successfully extended to more complex attention variants, like tensor attention, while maintaining almost linear computational efficiency \cite{as24_iclr}. \cite{kll+25} demonstrates an almost linear time algorithm for accelerating VAR Transformer inference. Further innovations in this space include RoPE-based attention mechanisms\cite{as24_rope,chl+24_rope} and privacy-preserving cross-attention methods~\cite{lssz24_dp}. The conv-basis approach introduced in \cite{lls+24_conv} offers another pathway to enhance attention computation speed, providing complementary solutions to this performance constraint. Researchers have also investigated various pruning-based techniques to optimize attention mechanisms \cite{lls+24_prune,cls+24,llss24_sparse,ssz+25_prune,ssz+25_dit,hyw+23,whl+24,xhh+24,ssz+25_prune}.

\paragraph{Gradient Approximation.}
The low-rank approximation is a widely utilized approach for optimizing transformer training by reducing computational complexity \cite{lss+24,lssz24_tat,as24b,hwsl24,cls+24,lss+24_grad}. Building on the low-rank framework introduced in \cite{as23}, which initially focused on forward attention computation, \cite{as24b} extends this method to approximate attention gradients, effectively lowering the computational cost of gradient calculations. The study in \cite{lss+24} further expands this low-rank gradient approximation to multi-layer transformers, showing that backward computations in such architectures can achieve nearly linear time complexity. Additionally, \cite{lssz24_tat} generalizes the approach of \cite{as24b} to tensor-based attention models, utilizing forward computation results from \cite{as24_iclr} to enable efficient training of tensorized attention mechanisms. Lastly, \cite{hwsl24} applies low-rank approximation techniques during the training of Diffusion Transformers (DiTs), demonstrating the adaptability of these methods across various transformer-based architectures.